\documentclass{article} 
\usepackage{nips14submit_e,times}
\usepackage{url}
\usepackage{amsfonts}
\usepackage{amssymb,amsmath}
\usepackage{amsthm}
\usepackage{graphicx}
\usepackage{color}
\usepackage{caption}
\usepackage{subcaption}
\usepackage{placeins}
\usepackage{enumitem}
\usepackage{wrapfig}
\usepackage{mathtools}

\newtheorem{definition}{Definition}
\newtheorem{lemma}{Lemma}
\newtheorem{Theorem}{Theorem}

\newtheorem{statement}{Statement}

\newtheorem{proposition}{Proposition}
\newtheorem*{proposition*}{Proposition}

\newcommand{\indep}{\rotatebox[origin=c]{90}{$\models$}}
 \newcommand{\ev}{E}
\newcommand{\iid}{\overset{i.i.d.}{\sim}}

\title{A Wild Bootstrap for Degenerate Kernel Tests}

\author{
Kacper Chwialkowski \\
Department of Computer Science\\
University College London\\
London, Gower Street, WC1E 6BT \\
\texttt{kacper.chwialkowski@gmail.com} \\
\And
Dino Sejdinovic \\
Gatsby Computational Neuroscience Unit, UCL \\
17 Queen Square, London WC1N 3AR \\
\texttt{dino.sejdinovic@gmail.com} \\
\AND
Arthur Gretton \\
Gatsby Computational Neuroscience Unit, UCL \\
17 Queen Square, London WC1N 3AR \\
\texttt{arthur.gretton@gmail.com} \\
}

\nipsfinalcopy 
\newcommand{\Hk}{\ensuremath{\mathcal{H}_k}}%
\begin{document}

\maketitle

\begin{abstract}

A wild bootstrap method for nonparametric hypothesis tests based on kernel distribution embeddings is proposed. This
  bootstrap method is used to construct provably consistent tests that apply to random
  processes, for which the naive permutation-based bootstrap
  fails. It applies to a large group of kernel tests
  based on V-statistics, which are degenerate under the null
  hypothesis, and non-degenerate elsewhere. To illustrate this
  approach, we construct a two-sample test, an instantaneous independence
  test and a multiple lag independence test for time series.  In experiments, the wild
  bootstrap gives strong performance on synthetic examples, on audio
  data, and in performance benchmarking for the Gibbs sampler. The code is available at 
  \url{https://github.com/kacperChwialkowski/wildBootstrap}.    

\end{abstract}

\vspace{-4mm}

\section{Introduction}
\vspace{-2.5mm}

Statistical tests based on distribution embeddings into reproducing kernel Hilbert spaces have been applied in many contexts,  including two sample testing \cite{HarBacMou08,gretton2012kernel,SugSuzItoKanetal11}, tests of independence \cite{gretton_kernel_2008,ZhaPetJanSch11,besserve_statistical_2013}, tests of conditional independence  \cite{fukumizu2007kernel,ZhaPetJanSch11}, and tests for higher order (Lancaster) interactions \cite{sejdinovic2013kernel}. %
For these tests,  consistency is guaranteed if and only if the observations are independent and identically distributed. 
Much real-world data fails to satisfy the i.i.d. assumption: audio signals, EEG recordings, text documents, financial time series, and samples obtained when running Markov Chain Monte Carlo, all show  significant temporal dependence patterns.  

The asymptotic behaviour of kernel test statistics becomes quite different when temporal dependencies exist within
the samples.
In recent work on independence testing using the Hilbert-Schmidt Independence
Criterion (HSIC) \cite{chwialkowski2014kernel}, the asymptotic distribution of the statistic under the null hypothesis is obtained for a pair of independent time series, which satisfy an absolute regularity or a $\phi$-mixing assumption.
In this case, the null distribution is shown to be an infinite weighted sum of {\em dependent} $\chi^2$-variables,
as opposed to the sum of \emph{independent} $\chi^2$-variables obtained in the i.i.d. setting \cite{gretton_kernel_2008}.
The difference in the asymptotic null distributions has important implications in practice:
under the i.i.d. assumption, an empirical estimate of the null distribution can be obtained by
repeatedly permuting the time indices of one of the signals. This breaks
the temporal dependence within the permuted signal, which causes the test to return an elevated
number of false positives, when used for testing time series. To address this problem, an alternative estimate of the null distribution
is proposed in \cite{chwialkowski2014kernel}, where the null distribution is simulated by repeatedly
{\em shifting} one
signal relative to the other. This preserves the temporal structure within each signal, while breaking the cross-signal
dependence.

A serious limitation of the shift procedure in \cite{chwialkowski2014kernel} is that it is specific
to the problem of independence testing: there is no obvious way to generalise it to  other
testing contexts. For instance, we might have two time series, with the goal of comparing
their marginal distributions - this is a generalization of the two-sample setting to which the shift
approach does not apply.

We note, however, that many kernel tests have a test statistic with a particular structure:  the Maximum Mean Discrepancy (MMD), HSIC, and the Lancaster interaction statistic,
each have empirical estimates which can be cast as normalized $V$-statistics,
$\frac{1} {n^{m-1}} \sum_{1\leq i_1,...,i_m \leq n} h(Z_{i_1},...,Z_{i_m})$,
where $Z_{i_1},...,Z_{i_m}$ are samples from a random process at the time points $\{i_1,\ldots,i_m\}$. We show that
a method of external randomization known as the {\em wild bootstrap} may be applied  \cite{leucht_dependent_2013,Shao2010} to simulate from the null distribution.
In brief, the arguments of the above sum are repeatedly multiplied by random, user-defined time series. For a test of level
$\alpha$, the $1-\alpha$ quantile of the empirical distribution obtained using these perturbed statistics serves as the test threshold. This approach has the important advantage over \cite{chwialkowski2014kernel} that it may be applied to {\em all} kernel-based tests for which $V$-statistics are employed, and not just for independence tests.

The main result of this paper is to show that the wild bootstrap procedure yields consistent tests for time series, i.e., tests based on the wild bootstrap  have a Type I error rate (of wrongly rejecting the null hypothesis) approaching the design parameter $\alpha$, and a Type II error (of wrongly accepting the null) approaching zero, as the number of samples increases. We use this result to construct a two-sample test using MMD, and an independence test using HSIC. The latter procedure is applied both to testing for instantaneous independence, and to testing for independence across multiple time lags, for which the earlier shift procedure of \cite{chwialkowski2014kernel} cannot be applied.





We begin our presentation in Section \ref{sec:background}, with a review of the $\tau$-mixing assumption required of the time series, as well as of   $V$-statistics (of which MMD and HSIC are instances). We also introduce the form taken by the wild bootstrap. In Section \ref{sec:main}, we establish a general consistency result for the wild bootstrap procedure on $V$-statistics, which we apply to MMD and to HSIC in Section \ref{sec:mmd_hsic}. Finally, in Section \ref{sec:Experiments}, we present a number of empirical comparisons: in the two sample case, we test for differences in audio signals with the same underlying pitch, and present a performance diagnostic for the output of a Gibbs sampler (the MCMC M.D.); in the independence case, we test for independence of two time series sharing a common variance (a characteristic of econometric models), and compare against the test of \cite{besserve_statistical_2013} in the case where dependence may occur at multiple, potentially unknown lags. Our tests outperform both the naive approach which neglects the dependence structure within the samples, and the approach of \cite{besserve_statistical_2013}, when testing across multiple lags. Detailed proofs are found in the appendices.

\section{Background}\label{sec:background}
The main results of the paper are based around two concepts: $\tau$-mixing \cite{dedecker2007weak}, which describes the dependence within the time series, and  $V$-statistics \cite{serfling80}, which constitute our test statistics. In this section, we review these topics, and introduce the concept of wild bootstrapped $V$-statistics, which will be the key ingredient in our test construction.
\paragraph{$\tau$-mixing.} The notion of $\tau$-mixing is used to characterise weak dependence. It is a less restrictive alternative to classical mixing coefficients, and is covered in depth in \cite{dedecker2007weak}. Let $\{Z_t,\mathcal{F}_t\}_{t \in \mathbb{N}}$  be a stationary sequence of integrable random variables, defined on a probability space $\Omega$ with a probability measure $P$ and a natural filtration $\mathcal{F}_t$. The process  is called $\tau$-dependent if 
\begin{align*}
\tau(r) &= \sup_{l \in \mathbb{N}} \frac 1 l \sup_{ r \leq i_1 \leq ... \leq i_l} \tau( \mathcal F_0,(Z_{i_1},...,Z_{i_l}) )  \overset{r \to \infty}{\longrightarrow} 0,\;\text{where} \\
\tau(\mathcal{M},X) &=  \ev \left( \sup_{g \in \Lambda} \left| \int g(t) P_{X|\mathcal{M}}(dt) - \int g(t) P_X(dt) \right| \right)
\end{align*}
and $\Lambda$ is the set of all one-Lipschitz continuous real-valued functions on the domain of $X$. $\tau(\mathcal M,X)$ can be interpreted as the minimal $L_1$ distance between $X$ and $X^*$ such that $X \overset{d}{=}X^*$ and $X^*$ is independent of $\mathcal M \subset \mathcal F$. Furthermore, if $\mathcal F$ is rich enough, this $X^*$ can be constructed.

Note that this mixing definition differs from commonly used notion of $\beta$ mixing (or $\phi$ mixing),  which was required in the previous work \cite{chwialkowski2014kernel}. We describe in more detail how these notions of dependence are related in Appendix \ref{append:differentMixing}.
%
\paragraph{$V$-statistics.} The test statistics considered in this paper are always $V$-statistics. Given the observations $Z=\left\{Z_t\right\}_{t=1}^n$, a $V$-statistic of a symmetric function $h$ taking $m$ arguments is given by 
\begin{equation}
\label{def:Vstat}
V(h,Z) = \frac{1}{n^m} \sum_{i \in N^m} \nolimits h(Z_{i_1},...,Z_{i_m}),
\end{equation}
where $N^m$ is a Cartesian power of a set $N= \{1,...,n\}$. For simplicity, we will often drop the second argument and write simply $V(h)$. 

We will refer to the function $h$ as to the \emph{core} of the $V$-statistic $V(h)$. While such functions are usually called kernels in the literature, in this paper we reserve the term kernel for positive-definite functions taking two arguments. A core $h$ is said to be $j$-degenerate if for each $z_1,\ldots,z_j$ $\ev h(z_1,\ldots , z_j , Z_{j+1}^*,\ldots ,Z_m^*) = 0,$ where $Z_{j+1}^*,\ldots,Z_m^*$ are independent copies of $Z_1$. If $h$ is $j$-degenerate for all $j\leq m-1$, we will say that it is \emph{canonical}. For a one-degenerate core $h$, we define an auxiliary function $h_2$, called the second component of the core, and given by $h_2(z_1,z_2) = \ev h(z_1,z_2, Z_3^*,\ldots, Z_m^*).$ Finally we say that $nV(h)$ is a normalized $V$-statistic, and that a $V$-statistic with a one-degenerate core is a degenerate $V$-statistic.  This degeneracy is common to many kernel statistics when the null hypothesis holds \cite{gretton2012kernel,gretton_kernel_2008,sejdinovic2013kernel}.

Our main results will rely on the fact that $h_2$ governs the asymptotic behaviour of normalized degenerate $V$-statistics. Unfortunately, the limiting distribution of such $V$-statistics is quite complicated - it is an infinite sum of \emph{dependent} $\chi^2$-distributed random variables, with a dependence  determined by the temporal dependence structure within the process $\{Z_t\}$ and by the eigenfunctions of a certain integral operator associated with $h_2$ \cite{i._s._borisov_orthogonal_2009,chwialkowski2014kernel}. Therefore, we propose a bootstrapped version of the $V$-statistics which will allow a consistent approximation of this difficult limiting distribution.  

\paragraph{Bootstrapped $V$-statistic.} 
We will study two versions of the bootstrapped $V$-statistics  
\begin{align}
 B_{1,n}(h,Z) = \frac{1}{n^m} \sum_{i \in N^m} \nolimits W_{i_1,n} W_{i_2,n} h(Z_{i_1},...,Z_{i_m}), \label{Vb1}\\ 
 B_{1,n}(h,Z) = \frac{1}{n^m} \sum_{i \in N^m}  \nolimits \tilde W_{i_1,n}  \tilde W_{i_2,n} h(Z_{i_1},...,Z_{i_m}),\label{Vb2}
\end{align}
where $\{W_{t,n}\}_{1 \leq t \leq n }$ is an auxiliary wild bootstrap process and $\tilde W_{t,n} = W_{t,n} - \frac 1 n \sum_{j=1}^n W_{j,n}$. This auxiliary process, proposed by \cite{Shao2010,leucht_dependent_2013}, satisfies the following assumption:

\emph{Bootstrap assumption:} $\{W_{t,n}\}_{1 \leq t \leq n }$ is a row-wise strictly stationary triangular array independent of all $Z_t$ such that $\ev W_{t,n}=0$ and $\sup_{n} \ev|W_{t,n}^{2+\sigma}| < \infty$ for some $\sigma > 0$. The autocovariance of the process is given by $\ev W_{s,n} W_{t,n}=\rho(|s-t|/l_n)$ for some function $\rho$, such that $\lim_{u \to 0} \rho(u) = 1$ and $\sum_{r=1}^{n-1} \rho(|r|/l_n)= O(l_n)$. The sequence $\left\{l_n\right\}$ is taken such that $l_n=o(n)$ but $\lim_{n \to \infty} l_n = \infty$. The variables $W_{t,n}$  are $\tau$-weakly dependent with coefficients $\tau(r) \leq C \zeta^{\frac{r} {l_n}}$ for $r=1,...,n$, $\zeta \in (0,1)$ and $C\in\mathbb R$.

As noted in in \cite[Remark 2]{leucht_dependent_2013}, a simple realization of a process that satisfies this assumption is $W_{t,n} = e^{-1/l_n}W_{t-1,n} + \sqrt{1 -e^{-2/l_n}} \epsilon_t$
where $W_{0,n}$ and $\epsilon_1,\ldots,\epsilon_n$ are independent standard normal random variables. For simplicity, we will drop the index $n$ and write $W_t$ instead of $W_{t,n}$. A process that fulfils the \emph{bootstrap assumption} will be called  \emph{bootstrap process}. Further discussion of the wild bootstrap is provided in the Appendix  \ref{wildintro}.
The versions of the bootstrapped $V$-statistics in \eqref{Vb1} and \eqref{Vb2} were previously studied in \cite{leucht_dependent_2013} for the case of canonical cores of degree $m=2$. We extend their results to higher degree cores (common within the kernel testing framework), which are not necessarily one-degenerate. When stating a fact that applies to both $B_1$ and $B_2$, we will simply write $B$, and the argument $h$ or index $n$ will be dropped when there is no ambiguity. 
\section{Asymptotics of wild bootstrapped $V$-statistics}\label{sec:main}
In this section, we present main Theorems that describe asymptotic behaviour of $V$-statistics, all the poofs are in the Appendix \ref{sec:Wildproofs}.  In the next section, these results will be used to construct kernel-based statistical tests applicable to dependent observations. Tests are constructed so that the  $V$-statistic is degenerate under the null hypothesis and non-degenerate under the alternative. Theorem \ref{th:mainOne} guarantees that the bootstrapped $V$-statistic will converge to the same limiting null distribution as the simple $V$-statistic. 

Throughout this paper we will make one mild assumption
\[
  \sup_{i \in N^m}\ev h(Z_i)^2 < \infty,
 \]
 where $Z_i = (Z_{i_1},\cdots Z_{i_m})$. This assumption is almost always automatically satisfied, since most of the kernels used in practice are bounded.
\begin{Theorem}
\label{th:mainOne}
Assume that the stationary process $Z_t$ is $\tau$-dependent with $\sum_{r=1}^\infty r^2 \sqrt{\tau(r)} < \infty$. If the core $h$ is a Lipschitz continuous, one-degenerate and its $h_2$-component is a positive definite kernel, such that $\ev h_2(Z_0,Z_0) < \infty$, then $nB_n$ \eqref{Vb1}, \eqref{Vb2},  and $n V_n$  \eqref{def:Vstat} converge weakly to the same distribution $V$.  Moreover $nB_n(h_2)$ and $nV_n(h_2)$ converge weakly to $\binom {m} {2} ^{-1} V$.
\end{Theorem}

On the other hand, if the $V$-statistic is not degenerate, which is usually true under the alternative, it converges to some non-zero constant. 
\begin{Theorem}
\label{th:mainThree}
Assume that the stationary process $Z_t$ is $\tau$-dependent with $\tau(r) = o(r^{-4})$. If the core $h$ is a Lipschitz continuous, and  $h_0$ component is positive then  $V_n$ converges in mean squared to $h_0$.
\end{Theorem}
In this setting, Theorem \ref{th:mainTwo} guarantees that the bootstrapped $V$-statistic will converge to zero in probability. This property is necessary in testing, as it implies that the test thresholds computed using the bootstrapped $V$-statistics will also converge to zero, and so will the corresponding Type II error.   
\begin{Theorem}
\label{th:mainTwo}
Assume that the stationary process $\{Z_t\}$ is $\tau$-dependent with a coefficient $\tau(r) = o(r^{-4})$. If the core $h$ is  a function of $m>1$ arguments then $B_1(h)$ and $o(n) B_2(h)$  converge to zero in mean squared. 
\end{Theorem}
Although both $B_2$ and $B_1$  converge to zero, the rate does not seem to be that same. As a consequence, tests that utilize $B_2$ usually give lower Type II error then the ones that use $B_1$. On the other hand, $B_1$ seems to better approximate $V$-statistic distribution under the null hypothesis. This agrees with our experiments in Section \ref{sec:Experiments} as well as with those in \cite[Section 5]{leucht_dependent_2013}).  
These results a sufficient for adopting kernel tests developed for i.i.d. data to tests that work on random processes. In particular Theorem \ref{th:mainOne}  justifies usage of bootstraped $V$-statistics for estimating quantiles of the null distribution, while Theorems \ref{th:mainThree}\ref{th:mainTwo} guarantee consistency.

The general testing procedure is 

\begin{itemize}
\item Calculate the test statistic $n V_{n}(h)$.

\item Obtain wild bootstrap samples $\{B_{n}(h)\}_{i=1}^{D}$
and estimate the $1-\alpha$ empirical quantile of these samples. 
\item If $n V_{n}(h)$ exceeds the quantile, reject.
\end{itemize}

\section{Applications to Kernel Tests}\label{sec:mmd_hsic}
In this section, we describe how the wild bootstrap for $V$-statistics can be used to construct kernel tests for independence and the two-sample problem, which are applicable to weakly dependent observations. We start by reviewing the main concepts underpinning the kernel testing framework.

For every symmetric, positive definite function, i.e., \emph{kernel} $k:\mathcal{X}\times\mathcal{X}\to\mathbb{R}$,
there is an associated reproducing kernel Hilbert space $\mathcal{H}_{k}$ \cite[p. 19]{BerTho04}.  The kernel embedding of a probability measure
$P$ on $\mathcal{X}$ is an element $\mu_{k}(P)\in\mathcal{H}_{k}$,
given by $\mu_{k}(P)=\int k(\cdot,x)\, dP(x)$ \cite{BerTho04,SmoGreSonSch07}.
If a measurable kernel $k$ is bounded, the mean embedding $\mu_{k}(P)$
exists for all probability measures on $\mathcal{X}$, and for many interesting
bounded kernels $k$, including the Gaussian, Laplacian and inverse
multi-quadratics, the kernel embedding $P\mapsto\mu_{k}(P)$ is injective.
Such kernels are said to be \emph{characteristic} \cite{SriGreFukLanetal10}.
The RKHS-distance $\left\Vert \mu_k(P_x)-\mu_k(P_y)\right\Vert_{{\mathcal H}_k}^2$ between embeddings of two probability measures $P_x$ and $P_y$
is termed the Maximum Mean Discrepancy (MMD), and its empirical version serves as a popular statistic for non-parametric two-sample testing \cite{gretton2012kernel}.
Similarly, given a sample of paired observations $\{(X_i,Y_i)\}_{i=1}^n\sim P_{xy}$, and kernels $k$ and $l$ respectively on $X$ and $Y$ domains, the RKHS-distance 
$\left\Vert \mu_\kappa(P_{xy})-\mu_\kappa(P_x P_y)\right\Vert_{{\mathcal H}_{\kappa}}^2$ between embeddings of the joint distribution and of the product of the marginals, measures dependence between $X$ and $Y$. Here, $\kappa((x,y),(x',y'))=k(x,x')l(y,y')$ is the kernel on the product space of $X$ and $Y$ domains.
This quantity is called Hilbert-Schmidt Independence Criterion (HSIC) \cite{gretton_measuring_2005,gretton_kernel_2008}. When characteristic RKHSs are used, the HSIC is zero iff $X \indep Y$: this follows from \cite{Asimplercondition}.
The  empirical statistic is written $\widehat{\text{HSIC}}_{\kappa} = \frac{1}{n^2}\text{Tr}(KHLH)$ for kernel matrices $K$ and $L$ and the centering matrix $H=I-\frac{1}{n}\mathbf{1}\mathbf{1}^\top$.

In this section, we describe how the wild bootstrap for $V$-statistics can be used to construct kernel tests for independence and the two-sample problem, in presence of weakly dependent observations.

\subsection{Wild Bootstrap For MMD}
Denote the observations by $\{X_i\}_{i=1}^{n_x}\sim P_x$, and $\{Y_j\}_{j=1}^{n_y}\sim P_y$. Our goal is to test the null hypothesis $\mathbf H_0: P_x=P_y$ vs. 
the alternative $\mathbf H_1: P_x\neq P_y$. In the case where samples have equal sizes, i.e., $n_x=n_y$, application of the wild bootstrap to MMD-based tests on dependent samples is straightforward: the empirical MMD can be written as a $V$-statistic with the core of degree two on pairs $z_i=(x_i,y_i)$ given by $h(z_1,z_2) = k(x_1,x_2)- k(x_1,y_2) - k(x_2,y_1) + k(y_1,y_2)$. It is clear that whenever $k$ is Lipschitz continuous and bounded, so is $h$. Moreover, $h$ is a valid positive definite kernel, since it can be represented as an RKHS inner product  $\left\langle k(\cdot, x_1) -k(\cdot, y_1),k(\cdot, x_2) -k(\cdot, y_2) \right\rangle_{\Hk}$. Under the null hypothesis, $h$ is also one-degenerate, i.e., $\ev h\left((x_1,y_1),(X_2,Y_2)\right) = 0$. Therefore, we can use the bootstrapped statistics in \eqref{Vb1} and \eqref{Vb2} to approximate the null distribution and attain a desired test level.

When $n_x\neq n_y$, however, it is no longer possible to write the empirical MMD
as a one-sample $V$-statistic. We will therefore require the following bootstrapped version of MMD
\begin{align}
\widehat{\text{MMD}}_{k,b}&=\frac{1}{n_x^2}\sum_{i=1}^{n_x}\sum_{j=1}^{n_x}\tilde W_i^{(x)}\tilde W_j^{(x)}k(x_i,x_j)-\frac{1}{n_x^2}\sum_{i=1}^{n_y}\sum_{j=1}^{n_y}\tilde W_i^{(y)}\tilde W_j^{(y)}k(y_i,y_j)\notag\\
{}&\qquad-\frac{2}{n_x n_y}\sum_{i=1}^{n_x}\sum_{j=1}^{n_y}\tilde W_i^{(x)}\tilde W_j^{(y)}k(x_i,y_j),\label{eq:mmdkb}
\end{align}
where $\tilde W_t^{(x)}=W_t^{(x)}-\frac{1}{n_x}\sum_{i=1}^{n_x}W_i^{(x)}$, $\tilde W_t^{(y)}=W_t^{(y)}-\frac{1}{n_y}\sum_{j=1}^{n_y}W_j^{(y)}$;  $\{W_t^{(x)}\}$ and $\{W_t^{(y)}\}$ are two auxiliary wild bootstrap processes that are independent of $\left\{ X_t \right\}$ and $\left\{ Y_t \right\}$ and also independent of each other, both satisfying the bootstrap assumption in Section \ref{sec:background}.  
The following Proposition shows that the bootstrapped statistic has the same asymptotic null distribution as the empirical MMD. The proof follows that of \cite[Theorem 3.1]{leucht_dependent_2013}, and is given in the Appendix \ref{sub:prop:mmd}.

\begin{proposition}\label{prop:mmd}
 Let $k$ be bounded and Lipschitz continuous, and let $\left\{ X_t \right\}$ and $\left\{ Y_t \right\}$ 
 both be $\tau$-dependent with coefficients $\tau(r) =  O(r^{-6-\epsilon})$, but independent of each other. Further, let $n_x=\rho_x n$ and $n_y=\rho_y n$ where $n=n_x+n_y$. Then, under the null hypothesis $P_x=P_y$, $\varphi\left(\rho_x \rho_y n\widehat{\text{MMD}}_k, \rho_x \rho_y n\widehat{\text{MMD}}_{k,b}\right)\to 0$ in probability as $n\to\infty$, where $\varphi$ is the Prokhorov metric and $\widehat{\text{MMD}}_k$ is the MMD between empirical measures.
\end{proposition}

\subsection{Wild Bootstrap For HSIC}\label{sec:hsic}
Using HSIC in the context of random processes is not new in the machine learning literature. For a 1-approximating functional of an absolutely regular process \cite{borovkova2001limit}, convergence in probability of the empirical HSIC to its population value was shown in \cite{smola_kernel_2008}. No asymptotic distributions were obtained, however, nor was a statistical test constructed.  The asymptotics of a normalized $V$-statistic were obtained in \cite{chwialkowski2014kernel}  for absolutely regular and $\phi$-mixing processes \cite{doukhan1994mixing}. Due to the intractability of the null distribution for the test statistic, the authors propose a procedure to approximate its null distribution using circular shifts of the observations leading to tests of instantaneous independence, i.e., of $X_t \indep Y_t$, $\forall t$. This was shown to be consistent under the null (i.e., leading to the correct Type I error), however consistency of the shift procedure under the alternative is a challenging open question (see \cite[Section A.2]{chwialkowski2014kernel} for further discussion).
 In contrast, the wild bootstrap guarantees test consistency under both hypotheses: null and alternative, which is a major advantage. 
In addition,  the wild bootstrap can be used in constructing a test for the harder problem of determining independence across multiple lags simultaneously, similar to the one in \cite{besserve_statistical_2013}.

Following symmetrisation, it is shown in \cite{gretton_kernel_2008,chwialkowski2014kernel} that the empirical HSIC can be written as a degree four $V$-statistic with core given by
\begin{align*}
h(&z_1,z_2,z_3,z_4) = \frac{1}{4!} \sum_{\pi \in S_4}  k(x_{\pi(1)},x_{\pi(2)}) [  l(y_{\pi(1)},y_{\pi(2)}) +  l(y_{\pi(3)},y_{\pi(4)}) - 2  l(y_{\pi(2)},y_{\pi(3)})],  
\end{align*}
where we denote by $S_n$ the group of permutations over $n$ elements. 
One-degeneracy of the core under the null hypothesis was stated in \cite[Theorem 2]{gretton_kernel_2008}, \cite[Section A.2, following eq. (11)]{gretton_kernel_2008} shows that $h_2$ is a kernel; $h_0\geq 0$ follows from the fact that HSIC is a distance. Using Theorems \ref{th:mainOne},\ref{th:mainTwo},\ref{th:mainThree} we can construct an independence test using $h$. Drawback of this test, when implemented in the most straightforward way,  is its quadruple computational complexity. To achieve quadratic time complexity, that matches time complexity of HSIC test for i.i.d. data, we modify our bootstrapped statistic.

\paragraph{Quadratic time HSIC.}
In this section we assume that kernels $k,l$ are positive and bounded. We define empirical mean embedding $\tilde \mu_X(x) = \frac 1 n \sum_{i}^n k(x,X_i) $ and centred kernels
\begin{equation*}
\begin{split}
\bar{k}(x,x') =& k(x',x) - \ev k(x,X) -\ev k(X',x') + \ev k(X,X')\\
=&\langle k(x,\cdot ) -\mu_X ,k(x',\cdot)- \mu_X\rangle. \\
 \tilde k(x,x') =& k(x,x') - \frac 1 n \sum_{i}^n k(x,X_i) - \frac 1 n \sum_{i}^n k(x',X_i) +   \frac {1} {n^2} \sum_{i,j}^n k(X_j,X_i)\\
=&\langle k(x,\cdot ) - \tilde\mu_X ,k(x',\cdot)- \tilde \mu_X\rangle. \\
\end{split}
\end{equation*}
where $X,X'$ are i.i.d. copies of $X_1$. Same definitions hold for the kernel $l$. Let $Q_i$ denote  $W_i$ or  $\tilde W_i$ (where  it is necessary, we check claims for both $W_i$ and  $\tilde W_i$ separately). We further define 
\begin{align}
\label{hsic_2}
S_n &= \frac {1} {\sqrt n}\sum_{i \in N} Q_i ( \phi(X_i) - \tilde \mu_X ) \otimes (\phi(Y_i )  - \tilde \mu_Y), \\
T_n &= \frac {1} {\sqrt n}  \sum_{i \in N} Q_i ( \phi(X_i) -  \mu_X ) \otimes(\phi(Y_i )  -  \mu_Y). 
\end{align}
First, we  relate $T_n$ to $B(h_2)$. 
\begin{statement}{ \cite[section A.2, following eq. (11)]{gretton_kernel_2008}}
\label{stm:thanks}
The second component of $h$ is $h_2(z_1,z_2) =  \frac 1 6 \bar{k}(x_1,x_2) \bar{l}(y_1,y_2).$ 
\end{statement}
\begin{lemma}
\label{lem:T_n}
 Squared norm of $T_n$  is equal to $ 6 B(h_2) $.
\end{lemma}
\begin{proof}
  \begin{align*}
 \| T_n \|^2 =  & \frac 1 n  \sum_{i,j \in N} Q_i Q_j \bigg\langle   ( \phi(X_i) -  \mu_X ) \otimes  (\phi(Y_i )  -  \mu_Y),   ( \phi(X_j) -  \mu_X ) \otimes (\phi(Y_j )  -  \mu_Y) \bigg\rangle  \\
  =& \frac 1 n \sum_{i,j \in N} Q_i Q_j \bar k(X_i,X_j)  \bar l(Y_i,Y_j)  \\
  =&6 B(h_2).
 \end{align*}
\end{proof}
Next we relate $S_n$ to $T_n$ -- we  show that the difference between them is asymptotically negligible. We start with a technical lemma.  
\begin{lemma}
\label{lem:hahaha}
If $(\bar k \times \bar k,Z_i)$  is of type $\varDelta$   of order $O(r^{-4})$ (see Definition \ref{def:varDelta}), then 
$$\lim_{n \to \infty} \ev \left \| \sqrt n (\tilde  \mu_X  - \mu_X) \right \|^4 = O(1).$$
\end{lemma}
\begin{proof}
 \begin{align*}
   \ev \left \| \sqrt n (\tilde  \mu_X  -  \mu_X) \right \|^4  &=  \ev \left \|  \frac {1} {\sqrt n}  \sum_{i \in N} \phi(X_i) - \mu_X  \right\|^4  \\
     &=\ev  \left( \frac 1 n   \sum_{i \in N} \langle  \phi(X_j) - \mu_X, \phi(X_i) - \mu_X \rangle  \right)^2 \\
   &=\frac{1}{n^2} \ev  \sum_{i \in N^4} \bar k \times \bar k (Z_i). 
 \end{align*}
 Since $(\bar k \times \bar k, X_i)$  is of type $\varDelta$, by Lemma \ref{lem:higherVstats}, the expected value is of order $O(1)$.
\end{proof}

\begin{lemma}
\label{lem:difS_nT_n}
 If $(\bar k \times \bar k,Z_i)$, $(\bar l \times \bar l,Z_i)$  are of type $\varDelta$ of order $O(r^{-4})$, then, under the null,  $\|S_n\|^2- \|T_n\|^2$  converges to zero in mean square. Under the alternative $\frac 1 n (\|S_n\|^2- \|T_n\|^2) $ converges to zero in  mean square.
\end{lemma}
\begin{proof}
We first show that $\ev \| S_n- T_n\|^2 \to 0$ both under the null and the alternative. Then,  using the fact that $\|T_n\|^2< \infty$ under the null  and $\frac 1 n \|T_n\|^2< \infty $ under alternative we  will conclude the proof. The difference $S_n- T_n$ is 
\begin{align*}
 \frac {1} { \sqrt n} &\sum_{i \in N}Q_i \bigg[ (\phi(X_i) - \tilde \mu_X ) \otimes (\phi(Y_i )  - \tilde \mu_Y) - ( \phi(X_i) -  \mu_X ) \otimes(\phi(Y_i )  -  \mu_Y)\bigg] \\
&= \frac {1} { \sqrt n} \sum_{i \in N} Q_i \bigg[    \phi(X_i)  \otimes   \mu_Y -\phi(X_i) \otimes \tilde \mu_Y \bigg]  \\
&+\frac {1} { \sqrt n} \sum_{i \in N} Q_i \bigg[    \phi(Y_i)  \otimes   \mu_X -\phi(Y_i) \otimes \tilde \mu_X \bigg] \\
&+\frac {1} { \sqrt n} \sum_{i \in N}Q_i  ( \tilde  \mu_X   \otimes \tilde \mu_Y - \mu_Y \otimes \mu_X).
\end{align*}
We  examine differences separately -- it is sufficient to show that each difference converges to zero in mean square. 

The expected  norm of the first difference  is
\begin{align*}
  \ev &\bigg \|\frac {1} { \sqrt n}\sum_{i \in N} Q_i \bigg[   \phi(X_i)  \otimes   \mu_Y -\phi(X_i) \otimes \tilde \mu_Y \bigg] \bigg \|^2\\
  &=\ev \bigg \| \sqrt n(  \mu_Y-  \tilde  \mu_Y) \otimes \frac {1} { \sqrt n} \sum_{i \in N} Q_i  \phi(X_i)  \bigg \|^2  \\
 & \leq \sqrt {\ev \bigg \| \sqrt n( \tilde \mu_Y - \mu_Y ) \bigg \|^4 \ev \bigg \| \frac {1} { \sqrt n} \sum_{i \in N} Q_i  \phi(X_i)  \bigg \|^4}.
 \end{align*}
We used $\|v \otimes u\| = \|v \|\| u\|$ and Cauchy-Schwarz inequality.  By Lemma \ref{lem:hahaha} the first term is $O(1)$. The second term is equal to 
 \begin{align*}
 \ev  \|  \frac {1} { n} \sum_{i \in N} Q_i  \phi(X_i)    \|^4 = \ev  \left(\frac {1} { n^2}  \sum_{i,j} k(X_i,X_j)Q_iQ_j \right)^2.
 \end{align*}
The expected value converges  to zero in mean square by  Lemma \ref{lem:higherVstats} (the assumption $\sup_{i,j}  k(X_i,X_j) < \infty$ is satisfied). Using similar reasoning, the second term 
\begin{align*}
 \ev &\bigg \|\frac {1} { \sqrt n}\sum_{i \in N} Q_i \bigg[  \phi(Y_i) \otimes \tilde \mu_X  -\phi(Y_i)  \otimes   \mu_X \bigg] \bigg \|^2
\end{align*}
 also converges to zero. The final term is 
 \begin{align*} 
 \ev  &\bigg \|  \frac {1} { \sqrt n} \sum_{i \in N}Q_i  ( \tilde  \mu_X   \otimes \tilde \mu_Y - \mu_Y \otimes \mu_X)\bigg \|^2  \\
 &=\ev \left| \frac {1} {  n} \sum_{i \in N}Q_i \right | \ev \bigg \|  \sqrt n (\tilde  \mu_X   \otimes \tilde \mu_Y - \mu_Y \otimes \mu_X )\bigg \|^2 \\
\end{align*}
$ \frac {1} { n} \sum_{i \in N}Q_i$ converges in  mean square to zero (Lemmas \ref{lem:meanWi}, \ref{stmt:obviousD}). We rewrite the second term  
\begin{align*}
 \ev \bigg \| \sqrt n (\tilde  \mu_X   \otimes \tilde \mu_Y -  \tilde\mu_Y \otimes \mu_X + \tilde \mu_Y \otimes \mu_X - \mu_Y \otimes \mu_X) \bigg \|^2 
 \end{align*}
It is sufficient to bound  
 \begin{align*}
 \ev \bigg \|  \sqrt n \tilde \mu_Y \otimes ( \tilde \mu_X   -\mu_X) \bigg \|^2 & \leq  \ev  \sqrt {\bigg \|  \tilde \mu_Y \bigg \|^4 \ev \bigg \| \sqrt n ( \tilde \mu_X   -\mu_X) \bigg \|^4}\\
   \ev \bigg \|  \sqrt n \mu_X \otimes ( \tilde \mu_Y   -\mu_Y) \bigg \|^2  &=    \bigg \|  \mu_X \bigg \|^2  \ev \bigg \|\sqrt n ( \tilde \mu_Y   -\mu_Y) \bigg \|^2  \\
\end{align*}
$ \ev   \| \tilde \mu_Y \|^4 = E \frac {1} {n^4} \sum_{i \in N^4} (l \times l)(Y_i) = O(1)$, since $ l$ is bounded. By Lemma \ref{lem:hahaha} $\ev \big \| \sqrt n ( \tilde \mu_X   -\mu_X) \big \|^4$ and $ \ev \big \|\sqrt n ( \tilde \mu_Y   -\mu_Y) \big \|^2$ are finite. Thus, to whole expression converges to zero. We proved  that $T_n-S_n$ converges in  mean square to zero.
We have 
\begin{align*}
  \ev | \|T_n\|^2 - \|S_n\|^2 |  & \leq\ev  \big | \|T_n\| - \|S_n\| \big | \big |  \|T_n\| +  \|S_n\| \big | \\
  &\leq \sqrt{ \ev \big | \|T_n\| - \|S_n\| \big |^2  \ev \big| \|T_n\| +  \|S_n\| \big|^2 }\\ 
\end{align*}
To show that the above expression converges to zero it is sufficient to show that  $\ev \|T_n\|^2< \infty$ and  $\ev  \|S_n\|^2< \infty$.
Under the null hypothesis, by Lemma \ref{lem:higherVstats}, expected value of   $\ev \|T_n\|^2 =n B_n(h_2) $ is finite. 
Since $\ev \| T_n -S_n\|^2 \to 0$ we also have  $\ev \| T_n -S_n\| \to 0$. 
Therefore we have 
\begin{align*}
 \ev \|S_n\|^2 &\leq \ev \|S_n-T_n+T_n\|^2 \\
 & \leq \ev \|S_n-T_n\|^2+ \ev \| T_n -S_n\| \ev \|T_n\| + \ev\|T_n\|^2 < \infty
\end{align*}
Under the alternative we have 
\begin{align*}
  n^{-1}  \ev | \|T_n\|^2 - \|S_n\|^2 |  &\leq n^{-1}  \ev \big | \|T_n\| - \|S_n\| \big | \big |   \|T_n\| +  \|S_n\| \big | \\
  &\leq \sqrt{ \ev \big | \|T_n\| - \|S_n\| \big |^2   n^{-1}\ev \big| \|T_n\| +  \|S_n\| \big|^2 }\\ 
\end{align*}
it is sufficient to show that  $n^{-1}  \ev\|T_n\|^2< \infty$ and  $n^{-1}  \ev \|S_n\|^2< \infty$. By Theorem \ref{th:mainTwo}, $n^{-1}  \ev \|T_n\|^2< \infty$ is finite and, using the reasoning similar to the one above, we have that    $n^{-1}  \ev \|S_n\|^2< \infty$.
\end{proof}

This shows  that we can use squared norm of $S_n$ as a bootstrapped test statistic.  For HSIC we redefine $B_n$ 
\begin{align}
\label{eq:hsic:1}
B_n^* :=  \| S_n \|^2 = \frac 1  n \sum_{i,j \in N }Q_i, Q_j \tilde k(X_i,X_j) \tilde l(X_i,X_j).
\end{align}
$B_{1}^*$ corresponds to $Q_i=W_i$, $B_{2}^*$  corresponds to $Q_i= \tilde W_i$ . This bootstrapped statistic interestingly coincides with $V_n(h)$. \cite{gretton_kernel_2008} showed that 
\begin{align}
\label{eq:hsic:2}
V_n(h) = \frac 1  n \sum_{i,j \in N }\tilde k(X_i,X_j) \tilde l(X_i,X_j).
\end{align}
Finally, notice that both statistics \ref{eq:hsic:1} and \ref{eq:hsic:2} can be calculated in quadratic time.

\begin{proposition}
\label{prop:null}
Let  $Z_t=\left(X_t,Y_t\right)$  be a stationary process  that is $\tau$-dependent such that $\sum_{r=1}^{\infty} r^2 \sqrt{ \tau(r)} <\infty$. Under the null hypothesis  $B_n^*$ (\ref{eq:hsic:1}) and $n V_n(h)$  (\ref{eq:hsic:2})converge weakly to the same distribution. Under the alternative hypothesis $B_n^*$  converges to zero in probability, while $V_n(h)$ converges to a positive constant. 
\end{proposition}
\begin{proof}
We calculate 
\begin{align*}
 n V_n(h) - B_n^* = n V_n(h) - 6 nB_n(h_2) + 6 nB_n(h_2) -  B_n^* .
\end{align*}
 By Lemma \ref{lem:T_n},  $6 nB_n(h_2) = \| T_n \|^2$. By definition \eqref{eq:hsic:1}, $B_n^* =  \| S_n \|^2 $ .  By Lemma \ref{lem:difS_nT_n}, $6 nB_n(h_2) -  B_n^*$ converges to zero in mean square. We check assumptions; since process $Z_t$ is $\tau$-mixing (of order $o(r^{-4})$ ) and both $\bar k$, $\bar l$ are canonical, Lemma \ref{lem:disentangle} guarantees that  $(\bar k,Z_i)$, $(\bar l,Z_i)$  are of type $\varDelta$ of order $O(r^{-4})$.
 
 Under the null  hypothesis, by Theorem \ref{th:mainOne}, $n V_n(h) - 6 nB_n(h_2)$ converges to zero.  We check assumptions; by Lemma \ref{stm:thanks}, $h_2$ is a symmetric, one-degenerate, bounded kernel, assumptions concerning $\tau$-mixing are satisfied. 
 
 Under the alternative, by Theorem \ref{th:mainTwo} and  Lemma \ref{lem:difS_nT_n} respectively,  $6 B_n(h_2)$ and $\frac 1 n B_n^* -6 B_n(h_2) $ converge to zero in mean square. By Theorem \ref{th:mainTwo}, $V_n(h)$ converges to a positive constant.
 \end{proof}

We consider two types of tests: instantaneous independence and independence at multiple time lags.

\paragraph{Test of instantaneous independence}
 
Here, the null hypothesis  $\mathbf{H_0}$ is that  $X_t$ and $Y_t$ are independent at all times $t$,  and the alternative hypothesis $\mathbf{H_1}$ is that they are dependent. We use Proposition\ref{prop:null} directly to bootstrap an appropriate quantile and compare it with a test statistic.  

\paragraph{Lag-HSIC}
Proposition \ref{prop:null} allows us to construct a test of time series independence that is similar to one designed by  \cite{besserve_statistical_2013}. Here, we will be testing against a broader null hypothesis:  $X_t$ and $Y_{t'}$ are independent for $|t-t'|<M$ for an arbitrary large but fixed $M$. 

Since the time series $Z_t=(X_t,Y_t)$ is stationary, it suffices to check whether there exists a dependency between $X_t$ and $Y_{t+m}$ for $-M \leq m \leq M$. Since each lag corresponds to an individual hypothesis, we will require a Bonferroni correction to attain a desired test level $\alpha$. We therefore define $q = 1-\frac{\alpha}{2M+1}$. The shifted time series will be denoted $Z_t^m =(X_t,Y_{t+m})$. Let $S_{m,n}=n V_n(h,Z^m)$ denote the value of the normalized HSIC statistic calculated on the shifted process $Z_t^m$. Let $F_{b,n}$ denote the empirical cumulative distribution function obtained by the bootstrap procedure using $B_n^*$ (\ref{eq:hsic:1}). The test will then reject the null hypothesis if the event $\mathcal A_n = \left\{ \max_{-M \leq m \leq M} S_{m,n} > F^{-1}_{b,n}(q) \right\}$ occurs. By a simple application of the union bound, it is clear that the asymptotic probability of the Type I error will be $\lim_{n\to\infty}P_{\,\mathbf{H_0}}\left(\mathcal A_n\right)\leq\alpha$. On the other hand, if the alternative holds, there exists some $m$ with $|m|\leq M$ for which $V_n(h,Z^m)=n^{-1} S_{m,n}$ converges to a non-zero constant. In this case  
\begin{align}
\label{eg:aletrnative1}
&P_{\,\mathbf{H_1}}(\mathcal A_n)  \geq  P_{\,\mathbf{H_1}}( S_{m,n} > F^{-1}_{b,n}(q)) = P_{\,\mathbf{H_1}}( n^{-1} S_{m,n} > n^{-1} F^{-1}_{b,n}(q) ) \to 1
\end{align}
as long as $n^{-1} F^{-1}_{b,n}(q)\to 0$, which follows from the convergence of $B_n^*$ (\ref{eq:hsic:1}) to zero in probability shown in Proposition \ref{prop:null}. Therefore, the Type II error of the multiple lag test is guaranteed to converge to zero as the sample size increases.
Our experiments in the next Section demonstrate that while this procedure is defined over a finite range of lags, it results in tests  more powerful than the procedure for an infinite number of lags proposed in \cite{besserve_statistical_2013}. 
We note that a procedure that works for an infinite number of lags, although possible to construct, does not add much practical value under the present assumptions. Indeed,  since the $\tau$-mixing assumption applies to the joint sequence $Z_t=(X_t,Y_t)$, dependence between $X_t$ and $Y_{t+m}$ is bound to disappear at a rate of $o(m^{-6})$, i.e., the variables both within and across the two series are assumed to become gradually independent at large lags.

\begin{table}\caption{Rejection rates for two-sample experiments. {\bf MCMC}: sample size=500; a Gaussian kernel with bandwidth
$\sigma=1.7$ is used; every second Gibbs sample is kept (i.e., after a pass
through both dimensions). {\bf Audio}: sample sizes are $(n_x,n_y)=\{(300,200),(600,400),(900,600)\}$; a Gaussian kernel with bandwidth
$\sigma=14$ is used. {\bf Both}: wild bootstrap
uses blocksize of $l_n=20$; averaged over at least 200 trials. The Type II error for all tests was zero}
\label{tab:gibbs_mmd}
\centering{}%
\begin{tabular}{|c|c|c|c|c|c|}
\cline{2-6} 
\multicolumn{1}{c|}{} & {\footnotesize experiment $\backslash$ method} & {\footnotesize permutation} & {\footnotesize $\widehat{\text{MMD}}_{k,b}$} & {\footnotesize $V_{b1}$} & {\footnotesize $V_{b2}$}\tabularnewline
\hline 
\textbf{\scriptsize MCMC} & {\footnotesize i.i.d. vs i.i.d. ($\mathbf{H}_{0}$)} & {\small .040} & {\small .025} & {\small .012}\textbf{\small{} } & {\small .070}\tabularnewline
\cline{2-6} 
 & {\footnotesize i.i.d. vs Gibbs ($\mathbf{H}_{0}$)} & {\small .528 } & {\small .100} & {\small .052} & {\small .105}\tabularnewline
\cline{2-6} 
 & {\footnotesize Gibbs vs Gibbs ($\mathbf{H}_{0}$)} & {\small .680 } & {\small .110} & {\small .060} & {\small .100}\tabularnewline
\hline 
\textbf{\scriptsize Audio} & {\footnotesize $\mathbf{H}_{0}$} & {\small \{.970,.965,.995\}} & {\small \{.145,.120,.114\}} & \multicolumn{1}{c}{} & \multicolumn{1}{c}{}\tabularnewline
\cline{2-4} 
 & {\footnotesize $\mathbf{H}_{1}$} & {\small \{1,1,1\}} & {\small \{.600,.898,.995\}} & \multicolumn{1}{c}{} & \multicolumn{1}{c}{}\tabularnewline
\cline{1-4} 
\end{tabular}
\end{table}

\vspace{-2mm}
\section{Experiments}
\label{sec:Experiments}

\begin{figure}
\centering  
\includegraphics[width=397.5pt,height=102.5pt]{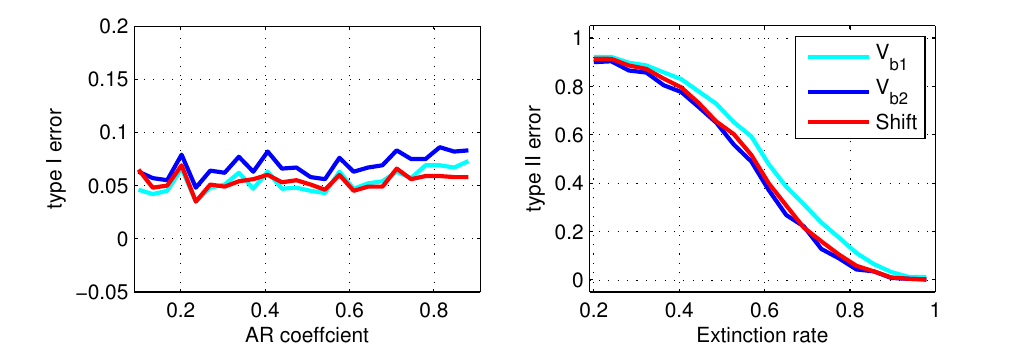}
\caption{Comparison of Shift-HSIC and tests based on $V_{b1}$ and $V_{b2}$. The left panel shows the performance under the null hypothesis, where a larger AR coefficient implies a stronger temporal dependence. The right panel show the performance under the alternative hypothesis, where a larger extinction rate implies a greater dependence between processes.}
\label{fig:arExtinct}
\end{figure}

\begin{figure}
\centering  
\includegraphics[width=397.5pt,height=102.5pt]{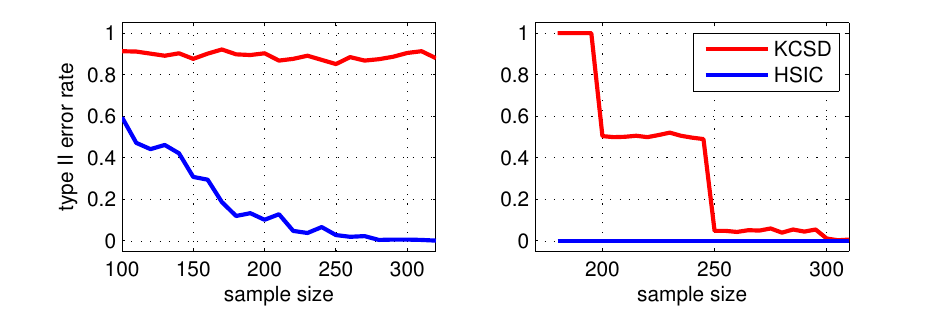}
\caption{In both panel Type II error is plotted. The left panel presents the error of the lag-HSIC and KCSD algorithms for a process following dynamics given by the equation \eqref{eq:dynamics2}.  The errors for a process with dynamics given by equations \eqref{eg:dymamics1a} and \eqref{eg:dymamics1b} are shown in the right panel. The X axis is indexed by the time series length, i.e., sample size. The Type I error was around 5\%.}
\label{fig:phaseAndVar}
\end{figure}

\vspace{-4mm}
\paragraph{The MCMC M.D.}
We employ MMD in order to diagnose how far an MCMC chain is from its stationary distribution \cite[Section 5]{sejdinovic_KAMH}, 
by comparing the MCMC sample to a benchmark sample. A hypothesis test of whether the sampler has converged based on the standard permutation-based bootstrap leads to too many rejections of the null hypothesis, due to  dependence within the chain. Thus, one would require heavily thinned chains, which is wasteful of samples and computationally burdensome.
Our experiments indicate that the wild bootstrap approach allows consistent tests directly on the chains, as it attains a desired number of false positives.\\
To assess performance of the wild bootstrap in determining MCMC convergence, we consider the situation where samples $\{X_i\}$ and $\{Y_i\}$ are bivariate, and both have the identical marginal distribution given by an elongated normal
$P=\mathcal{N}\left(\left[\protect\begin{array}{cc}
0 & 0\protect\end{array}\right],\left[\protect\begin{array}{cc}
15.5 & 14.5\protect\\
14.5 & 15.5
\protect\end{array}\right]\right)$.
However, they could have arisen either as independent samples, or as outputs of the Gibbs sampler with stationary distribution $P$. 
Table \ref{tab:gibbs_mmd} shows the \emph{rejection rates} under the significance level $\alpha=0.05$. It is clear that in the case where at least one of the samples is a Gibbs chain, the permutation-based test has a Type I error much larger than $\alpha$. 
The wild bootstrap using $V_{b1}$ (without artificial degeneration) yields the correct Type I error control in these cases. Consistent with findings in \cite[Section 5]{leucht_dependent_2013}, $V_{b1}$ mimics the null distribution better than $V_{b2}$. The bootstrapped statistic $\widehat{\text{MMD}}_{k,b}$ in \eqref{eq:mmdkb} which also relies on the artificially degenerated bootstrap processes, behaves similarly to $V_{b2}$.
In the alternative scenario where $\{Y_i\}$ was taken from a distribution with the same covariance structure but with the mean set to $\mu= \left[\protect\begin{array}{cc}
2.5 & 0\protect\end{array}\right]$, the Type II error for all tests was zero.
\vspace{-4mm}
\paragraph{Pitch-evoking sounds}
Our second experiment  is a two sample test on sounds studied in the field of pitch perception \cite{hehrmannthesis}. We synthesise the sounds with the fundamental frequency
parameter of treble C, subsampled at 10.46kHz. Each $i$-th period
of length $\Omega$ contains $d=20$ audio samples at times $0=t_{1}<\ldots<t_{d}<\Omega$
-- we treat this whole vector as a single observation $X_{i}$ or
$Y_{i}$, i.e., we are comparing distributions on $\mathbb{R}^{20}$.
Sounds are generated based on the AR process $a_{i}=\lambda a_{i-1}+\sqrt{1-\lambda^{2}}\epsilon_{i}$,
where $a_{0},\epsilon_{i}\sim\mathcal{N}(0,I_{d})$, with $X_{i,r}=\sum_{j}\sum_{s=1}^{d}a_{j,s}\exp\left(-\frac{\left(t_{r}-t_{s}-(j-i)\Omega\right)^{2}}{2\sigma^{2}}\right)$.
Thus, a given pattern -- a smoothed version of $a_{0}$ -- slowly
varies, and hence the sound deviates from
periodicity, but still evokes a pitch. We take
$X$ with $\sigma=0.1\Omega$
and $\lambda=0.8$, and $Y$ is either an independent copy of $X$
(null scenario), or has $\sigma=0.05\Omega$ (alternative scenario)
(Variation in the smoothness parameter changes the width of the
spectral envelope, i.e., the brightness of the sound). $n_x$ is taken to be different from $n_y$. Results in Table \ref{tab:gibbs_mmd} demonstrate that the
approach using the wild bootstrapped statistic in \eqref{eq:mmdkb} allows control of the Type I error and reduction of the Type II error with increasing sample size, while the permutation test
virtually always rejects the null hypothesis.
As in \cite{leucht_dependent_2013} and the MCMC example, the artificial degeneration of the wild bootstrap process causes the Type I error to remain above the design parameter of $0.05$, although it can be observed to drop with increasing sample size.

%

\vspace{-4mm}
\paragraph{Instantaneous independence}
To examine instantaneous independence test performance, we compare it with the Shift-HSIC procedure \cite{chwialkowski2014kernel} on the 'Extinct Gaussian' autoregressive process proposed in the \cite[Section 4.1]{chwialkowski2014kernel}. Using exactly the same setting we compute type I error as a function of the temporal dependence and type II error as a function of extinction rate. Figure \ref{fig:arExtinct} shows that all three tests (Shift-HSIC and tests based on $V_{b1}$ and $V_{b2}$) perform similarly.   
\vspace{-0.2cm}
\paragraph{Lag-HSIC}
The KCSD  \cite{besserve_statistical_2013} is, to our knowledge, the only test procedure to reject the null hypothesis if there exist $t$,$t'$ such that $Z_t$ and $Z_{t'}$ are dependent. In the experiments, we compare lag-HSIC with KCSD on two kinds of processes: one  inspired by econometrics and one from \cite{besserve_statistical_2013}.\\ 
In lag-HSIC, the number of lags under examination was equal to $\max\{10,\log n\}$, where $n$ is the sample size. We used Gaussian kernels with widths estimated by the median heuristic. The cumulative distribution of the $V$-statistics was approximated by samples from $n V_{b2}$. To model the tail of this distribution, we have fitted the generalized Pareto distribution to the bootstrapped samples (\cite{pickands1975statistical} shows that for a large class of underlying distribution functions such an approximation is valid).\\
The first process is a pair of two time series which share a common variance,   
\begin{align}
\label{eq:dynamics2}
 X_t = \epsilon_{1,t} \sigma_t^2, \quad  Y_t = \epsilon_{2,t}  \sigma_t^2,  \sigma_t^2 = 1 + 0.45(X_{t-1}^2 + Y_{t-1}^2 ), \quad \epsilon_{i,t} \iid \mathcal N(0,1), \quad i \in \{1,2\} .
\end{align}
The above set of equations is an instance of the VEC dynamics \cite{bauwens_multivariate_2006} used in econometrics to model market volatility. The left panel of the Figure \ref{fig:phaseAndVar} presents the Type II error rate: for KCSD it remains at 90\% while for lag-HSIC it gradually drops to zero. The Type I error, which we calculated by sampling two independent copies $(X^{(1)}_{t},Y^{(1)}_{t})$ and $(X^{(2)}_{t},Y^{(2)}_{t})$ of the process and performing the tests on the pair $(X^{(1)}_{t},Y^{(2)}_{t})$, was around 5\% for both of the tests.\\
Our next experiment is a process sampled according to the dynamics proposed by \cite{besserve_statistical_2013},      
\begin{alignat}{3}
  \quad X_t &= \cos(\phi_{t,1}),   &\quad  \phi_{t,1} &= \phi_{t-1,1} + 0.1\epsilon_{1,t} + 2 \pi f_1 T_s, &\quad  \epsilon_{1,t}  \iid \mathcal N(0,1),  \label{eg:dymamics1a} \\  
  Y_t &= [2+C\sin(\phi_{t,1})]\cos(\phi_{t,2}),  &   \phi_{t,2} &= \phi_{t-1,2} + 0.1\epsilon_{2,t} + 2 \pi f_2 T_s , &  \epsilon_{2,t}  \iid \mathcal N(0,1), \label{eg:dymamics1b}
\end{alignat}
with parameters $C=.4$, $f_1=4Hz$,$f_2=20Hz$, and frequency $\frac {1} {T_s} = 100 Hz$. We compared performance of the KCSD algorithm, with parameters set to vales recommended in \cite{besserve_statistical_2013}, and the lag-HSIC algorithm. The Type II error of lag-HSIC, presented in the right panel of the Figure \ref{fig:phaseAndVar}, is substantially lower than that of KCSD. The Type I error ($C=0$) is equal or lower than 5\% for both procedures. Most oddly, KCSD error seems to converge to zero in steps. This may be due to the method relying on a spectral decomposition of the signals across a fixed set of bands. As the number of samples increases, the quality of the spectrogram will improve, and dependence will become apparent in bands where it was undetectable at shorter signal lengths.

\small
\bibliographystyle{plain}
\bibliography{nips14Bib}

\newpage
\normalsize
\appendix

\section{An Introduction to the Wild Bootstrap}
\label{wildintro}
Bootstrap methods aim to evaluate the accuracy of the sample estimates - they are particularly useful when dealing with complicated distributions, or when the assumptions of a parametric procedure are in doubt. Bootstrap methods randomize the dataset used for the sample estimate calculation, so that a new dataset with a similar statistical properties is obtained, e.g. one popular method is resampling. In the wild bootstrap method  the observations in the dataset are multiplied by  appropriate random numbers. To present the idea behind the wild bootstrap we will discuss a toy example similar to that in \cite{Shao2010}, and then relate it to the wild bootstrap method used in this article. 

Consider a stationary autoregressive-moving-average random process $\{X_i\}_{i \in \mathbf{Z}}$ with zero mean. The normalized sample mean of the process $X_t$ has normal distribution
\begin{equation}
\frac{\sum_{i=1}^{N} X_i}{\sqrt{n}} \overset{d}{\to} N(0,\sigma_{\infty}^{2}),    
\end{equation}      
where $\sigma_{\infty}^2 = \sum_{j=-\infty}^{j=\infty} cov(X_0,X_j)$. The variance $\sigma_{\infty}^2$ is not easy to estimate (the naive approach of approximating different covariances separately and summing them has several drawbacks, e.g. how many empirical covariances should be calculated?). Using the wild bootstrap method we will construct processes that mimic behaviour of the $X_t$ process and then use them to approximate the distribution of the normalized sample mean, $\frac{\sum_{i=1}^{N} X_i}{\sqrt{n}}$. The bootstrap process used to to randomize the sample meets the following criteria: 
\begin{itemize}
\item $\{W_{t,n}\}_{1 \leq t \leq n }$ is a row-wise strictly stationary triangular array independent of all $X_t$, such that $\ev W_{t,n}=0$ and $\sup_{n} \ev|W_{t,n}^{2+\sigma}| < \infty$ for some $\sigma > 0$. 
\item The autocovariance of the process is given by $\ev W_{s,n} W_{t,n}=\rho(|s-t|/l_n)$ for some function $\rho$, such that $\lim_{u \to 0} \rho(u) = 1$. 
\item The sequence $\left\{l_n\right\}$ is taken such that $\lim_{n \to \infty} l_n = \infty$.
\end{itemize}
A process that fulfils those criteria, given also in the main article, is
\begin{align}
W_{t,n} = e^{-1/l_n}W_{t-1,n} + \sqrt{1 -e^{-2/l_n}} \epsilon_t
\end{align} 
  
We need to show that the distribution of the normalized sample mean of the process  $Y_t^{n} = W_t^{n}X_t$, where $|t| \leq n$, mimics the distribution $N(0,\sigma_{\infty}^2)$. It suffices to calculate the expected value and correlations:   
\begin{align}
\ev Y_t^{n} &= \ev W_t^n X_t = 0 ,\\
cov(Y_0^n,Y_t^n) &= cov(X_0,X_t)cov(Y_0^n,Y_t^n) = cov(X_0,X_t)\rho(|t|/l_n) \to cov(X_0,X_t)
\end{align}
The asymptotic auto-covariance structure of the process $Y_t$ is the same as the auto-covariance structure of the process $X_t$. Therefore 
\begin{equation}
\frac{\sum_{i=1}^{N} Y_i}{\sqrt{n}} \overset{d}{\to} N(0,\sigma_{\infty}).    
\end{equation}  

This mechanism is used in \cite{leucht_dependent_2013}. Recall that, under some assumptions, a normalized V-statistic can be written as 
$$
\sum_{k=0}^{\infty} \lambda_k  \left( \frac{ \sum_{i=1}^{n} \phi_k(X_i) } {\sqrt n}  \right)^2 \overset{P}{=} \frac 1  n \sum_{1\leq i,j \leq n} h(X_i,X_j) 
$$ 

where $\lambda_k$ are eigenvalues and $\phi_k$ are eigenfunction of the  kernel $h$, respectively.  Since $\ev  \phi_k(X_i) = 0$ (degeneracy condition) one may replace  
$$  \frac{ \sum_{i=1}^{n} \phi_k(X_i)} {\sqrt n} $$
with a bootstrapped version 
$$ \frac{  \sum_{i=1}^{n}  W_t^n \phi_k(X_i) } {\sqrt n}, $$  
and conclude, as in the toy example, that the limiting distribution of the single component of the sum $\sum_k \lambda_k  ...$  remains the same. One of the main  contributions of \cite{leucht_dependent_2013}  is in showing that the distribution of the whole sum $\sum_k \lambda_k \left(\frac{  \sum_{i=1}^{n}  W_t^n \phi_k(X_i) } {\sqrt n} \right)^2$ with the components bootstrapped  
converges in a particular sense (in  probability in Prokhorov metric) to the distribution of the normalized V-statistic, $\frac 1  n \sum_{1\leq i,j \leq n} h(X_i,X_j) $.


\section{Relation between $\beta$,$\phi$ and $\tau$ mixing}
\label{append:differentMixing}

Strong mixing is historically the most studied type of temporal dependence -- a lot of models, example being Markov Chains, are proved to be strongly mixing, therefore it's useful to relate weak mixing  to strong mixing. For a random variable $X$ on a probability space $(\Omega,\mathcal{F},P_X)$ and $\mathcal{M} \subset \mathcal{F}$ we define 
\begin{equation*}
\beta(\mathcal{M},\sigma(X)) = \| \sup_{A \in \mathbb{B}(R)} | P_{X|\mathcal{M}}(A) - P_X(A)|\|_1.
\end{equation*}
A process  is called $\beta$-mixing or absolutely regular if  
\begin{align*}
\beta(r) &= \sup_{l \in \mathbb{N}} \frac 1 l \sup_{ r \leq i_1 \leq ... \leq i_l} \beta( \mathcal F_0,(X_{i_1},...,X_{i_l}) )  \overset{r \to \infty}{\longrightarrow} 0,\
\end{align*}
$\phi$ mixing is defined similarly
\begin{equation*}
\phi(\mathcal{M},\sigma(X)) = \| \sup_{A \in \mathbb{B}(R)} | P_{X|\mathcal{M}}(A) - P_X(A)|\|_\infty.
\end{equation*}
 By \cite{bradley_basic_2005} we have  $\beta(\mathcal{M},\sigma(X)) \leq \phi(\mathcal{M},\sigma(X)) $ .

\cite[Equation  7.6]{dedecker2005new} relates $\tau$-mixing and $\beta$-mixing , as follows: if $Q_x$ is the generalized inverse of the tail function
\[
 Q_x(u) = \inf_{t \in R} \{  P(|X| > t) \leq u\},  
\]
then
\[
 \tau(\mathcal{M},X) \leq 2 \int_{0}^{\beta(\mathcal{M},\sigma(X))}  Q_x(u) du.
\]
While this definition can be hard to interpret, it can be simplified in the case $E|X|^p=M$  for some $p>1$, since via Markov's inequality $P(|X|>t) \leq \frac{M}{t^p}$, and thus $\frac{M}{t^p} \leq u $ implies $P(|X|>t) \leq u$. Therefore $Q'(u) = \frac{M}{\sqrt[p]{u}} \geq Q_x(u)$. As a result, we have the following inequality 
\begin{equation}
\label{eq:theBeta}
 \frac{ \sqrt[p]{ \beta(\mathcal{M},\sigma(X))} }{ M }  \geq C  \tau(\mathcal{M},X) 
\end{equation}

\paragraph{Models that satisfy $\tau$-mixing.}
\cite{dedecker2005new} provides  examples of systems that are tau-mixing. In particular, given that certain assumptions are satisfied  causal functions of stationary sequences, iterated random functions, Markov chains, expanding maps are all $\tau$-mixing. 

Of particular interest to this work are Markov chains. The assumptions provided by \cite{dedecker2005new}, under which Markov chains are tau-mixing are somehow difficult to check but we can use classical theorems about the $\beta$-mixing).  In particular \cite[Corollary 3.6]{bradley_basic_2005}  states that a Harris recurrent (chain returns to a fixed set of the state space an infinite number of times) and aperiodic Markov chain satisfies absolute regularity.  \cite[Theorem 3.7]{bradley_basic_2005} states that geometric ergodicity \footnote{ $\forall_{x} \|P^n(x,\cdot)  -\pi \|_{TV} \leq C q^n, 0<q<1$} implies geometric decay of the $\beta$ coefficient. Interestingly \cite[Theorem 3.3]{bradley_basic_2005} describes situations in which a non-stationary chain $\beta$-mixes exponentially. 

Using inequality \ref{eq:theBeta} between $\tau$-mixing coefficient and strong mixing coefficients one can use  those classical theorems show that e.g for  $p=2$ we have 
\[
 \sqrt{ \beta(\mathcal{M},\sigma(X))}  \geq \tau(\mathcal{M},X).
\]

\section{Proofs}
\label{sec:Wildproofs}
In this section we prove the main theorems. As for the notation,  $n$ denotes number of observations, $N = \{1,\cdots, n\}$, if $h$ is  function then $h \times h$ denotes  a  product of $h$ with itself, $\lim_{n \to \infty} X_n \overset{L_2}{=} X$ denotes convergence in mean square 

\subsection{Proof of the Theorem   \ref{th:mainOne}} 
\label{sec:prMainOne}

Hoeffding decomposition reduces any $V$-statistic to a sum of canonical $V$-statistics with canonical cores $h_c$, which are easier to study in context of  non-iid data. As an illustration, consider a canonical core $h$ of $m$ arguments and fix some indexes  $i_1 \leq  \cdots \leq i_{m-1} \ll i_m $, for a sake of  example we may assume that indexes represent time. If observations $Z_{i_1}, \cdots, Z_{i_{m-1}}$ are independent of the observation $Z_{i_m}$, then the expected value of $h(Z_{i_1}, \cdots, Z_{i_m})$, by degeneracy, is equal to zero. If it is reasonable to assume that $Z_{i_m}$ is almost independent of $Z_{i_1}, \cdots, Z_{i_{m-1}}$, maybe because it is so distant in time, then it is also reasonable to expect that for a canonical core $h$ (which is not too complicated ) 
\[
 \ev h(Z_{i_1}, \cdot, Z_{i_m}) \approx 0.
\]
which follows from the following approximate calculation
\[
 \int h(z_{i_1}, \cdot, z_{i_m}) dP_{Z_{i_1}, \cdots, Z_{i_m}} \approx  \int h(z_{i_1}, \cdots, z_{i_m}) dP_{Z_{i_1}, \cdots, Z_{i_{m-1}}} dP_{Z_{i_m}} =0
\]
We formalize this intuition.
\begin{definition}
 Associate with  any  set of indexes $ i_1,\cdots,i_m$ its nearest neighbor within the set. Suppose $i_r$ is is an index with the most distant nearest neighbor. We will call $i_r$ the most isolated index, and we will refer to its distance to the nearest neighbor as an isolation distance.
\end{definition}
Consider a following example, for the set $\{1,5,7\}$, $1$ is the most isolated index and the isolation distance is $4$.
\begin{definition}
\label{isolation}
\label{def:varDelta}
 Given a sequence of random variables $Z_{t} $ and a function $h$, if for all sets of indexes $i_1,\cdots,i_m$, with the isolation distance equal to $r$ 
 \[
  |E h(Z_{i_1}, \cdots, Z_{i_m})| \leq \varDelta(h,r)
 \]
 for some some function $\varDelta$, then we say that the pair $(h,Z_{t})$ is of type $\varDelta$. 
\end{definition}

The next theorem shows a growth rate of  a canonical $V$-statistic when a pair $h,Z_{t}$ is of type $\varDelta$. 
\begin{Theorem}
\label{th:boundTh}
Let $(Z_{t},h)$, where $h$ is a function of $m>1$ arguments, be a  of type $\varDelta$, with $\varDelta(h,r) = o( r^{-k})$ for some $k$, then
\begin{equation*}
\sum_{i\in N^{m}}\left| \ev  h(Z_i) \right| =  O\left(n^{\left\lfloor \frac{m}{2}\right\rfloor }\right)+ o\left(n^{2\left\lfloor \frac{m}{2}\right\rfloor +2-k}\right).
\end{equation*}
\end{Theorem}
\begin{proof} 
The proof uses a technique similar to   \cite[Lemma 3]{arcones1998law}.
 We will focus on ordered $m$-tuples $1\leq i_{1}\leq\ldots\leq i_{m}\leq n$,
and by considering all possible permutations of their indices, we
obtain an upper bound 
\begin{equation*}
\sum_{i\in N^{m}}\left|\ev  h\left(Z_{i_{1}},\ldots,Z_{i_{m}}\right)\right|  <  \sum_{1\leq i_{1}\leq\ldots\leq i_{m}\leq n}\sum_{\pi\in S_{m}}\left|  \ev h\left(Z_{i_{\pi(1)}},\ldots,Z_{i_{\pi(m)}}\right)\right|,
\end{equation*}

where (strict) inequality stems from the fact that the $m$-tuples
with some coinciding entries appear multiple times on the right.

Since  $(h,Z_{t})$ is a  of type $\varDelta$
\[
\forall i \in N^{m} \sum_{\pi\in S_{m}}\left| \ev h\left(Z_{i_{\pi(1)}},\ldots,Z_{i_{\pi(m)}}\right)\right| = O( \varDelta(h,w(i)) ),
\]
where $w(i)$ is an isolating distance of the index set $i = i_1,\cdots i_m$. We need to estimate order of the sum 
\begin{equation*}
 \sum_{1\leq i_{1}\leq\ldots\leq i_{m}\leq n}  O( \varDelta(h,w(i)) ).
\end{equation*}
Let us upper bound the number of ordered $m$-tuples $i$ with $w(i)=w$.  Denote $s=\left\lfloor \frac{m}{2}\right\rfloor +1$. 
$i_{1}$ can take $n$ different values, but since $i_{2}\leq i_{1}+w$,
$i_{2}$ can take at most $w+1$ different values.
For $2\leq l\leq s-1$, since $\min\left\{ i_{2l}-i_{2l-1},i_{2l-1}-i_{2l-2}\right\} \leq w$,
we can either let $i_{2l-1}$ take up to $n$ different values and
let $i_{2l}$ take up to $w+1$ different values (if $i_{2l}-i_{2l-1}\leq i_{2l-1}-i_{2l-2}$)
or let $i_{2l-1}$ take up to $w+1$ different values and let $i_{2l}$
take up to $n$ different values (if $i_{2l}-i_{2l-1}>i_{2l-1}-i_{2l-2}$),
upper bounding the total number of choices for $\left[i_{2l-1},i_{2l}\right]$
by $2n(w+1)$. Finally, the last term $i_{m}$ can always have at
most $w+1$ different values.  
This brings the total number of $m$-tuples with $w(i)=w$ to at most $2^{\ensuremath{s-2}}n^{s-1}(w+1)^{s}$.
Thus, the number of $m$-tuples with $w(i)=0$ is $O(n^{s-1})$ and
since $\ev  h\left(Z_{i_{1}},\ldots,Z_{i_{m}}\right) < \infty$, we have
\begin{align*}
 &  \sum_{1\leq i_{1}\leq\ldots\leq i_{m}\leq n}O( \varDelta(h,w(i)) )\\
 & \leq O(n^{s-1})+\sum_{w=1}^{n-1}\;\sum_{\underset{w(i)=w}{1\leq i_{1}\leq\ldots\leq i_{m}\leq n}:} O( \varDelta(h,w(i)) )\\
 & \leq O(n^{s-1})+ n^{s-1}\sum_{w=1}^{n-1}(w+1)^{s} O(\varDelta(h,w))\\
 & \leq O(n^{s-1})+n^{s-1}\sum_{w=1}^{n-1}o(w^{s-k})\\
 & \leq O(n^{s-1})+n^{s-1}\max( o(n^{s-k+1}),O(1))\\
  & \leq O(n^{s-1})+o(n^{2s-k})+ O(n^{s-1}))\\
 &=  O(n^{s-1})+o(n^{2s-k}),
\end{align*}
which proves the claim. We have used $\varDelta(h,w)=o(w^{-k})$. 
\end{proof}
The previous theorem states sufficient conditions for a $V$-statistic  or a bootstrapped $V$-statistic to  converge to zero.
\begin{lemma}
 \label{lem:higherVstats}
 Let $h$ be a function of  $m>1$ arguments and let $(\{ Z_{t}\}_{t \in N},h \times h)$ be a of type $\varDelta$, with $\varDelta(h \times h,r) = o( r^{-4})$. If  $\{ G_i \}_{i \in N}$  is a random process,  independent of $ Z_{t}$, such that $\sup_{i} \ev G_i^4 < \infty$, with notation $T_n =\frac {1} {n^{m-1}} \sum_{i \in N^m}   G_{i_1}G_{i_2}   h(Z_i)  $,
\begin{align*}
 \begin{cases}
 \lim_{n \to \infty}  o(1) T_n \overset{L_2}{=} 0 & m=2,  \\
\lim_{n \to \infty} T_n \overset{L_2}{=} 0  & m>2
\end{cases}
\end{align*}
since, 
\begin{align*}
 \begin{cases}
 \ev T_n^2 = O(1) & m=2,  \\
\ev T_n^2  = o(1)  & m>2. 
\end{cases}
\end{align*}
\end{lemma}
\begin{proof}
First we verify that for $i,j \in N^m$
\[
 a_{i,j} =  \ev G_{i_1} G_{i_2}  G_{j_1} G_{j_2}
\]
is uniformly bounded. We get the bound by applying Cauchy-Schwarz iteratively and using assumption   $\sup_{i} \ev G_i^4 < \infty$. 

We check that the second non-central moment converges to zero,
\begin{align*}
&\ev  \left( T_n \right)^2 \\
&= \frac {1} {n^{2m-2}}   \sum_{i,j \in N^{m}}    \ev G_{i_1} G_{i_2} G_{j_1} G_{j_2} \ev h(Z_i)h(Z_j)   \\
&\leq    \frac {1} {n^{2m-2}} \sum_{i,j \in N^{m}}  |a_{i,j} \ev h(Z_i)h(Z_j) |  &\\
&\leq \left( \sup_{n} \sup_{i,j \in N^m} |a_{i,j}| \right) \frac {1} {n^{2m-2}} \sum_{i,j \in N^{m}} |\ev h(Z_i)h(Z_j) |.
\end{align*}
Supremum over $n$ is needed since $\ev G_{i_1} G_{i_2}  G_{j_1} G_{j_2}$ might change with $n$. Lemma \ref{th:boundTh}, by the assumption that  $(h(\cdots ) \times h(\cdots ),Z_t )$ is of type $\varDelta$, the growth of the inner sum 
$
 \sum_{i,j \in N^{m}} |\ev h(Z_i)h(Z_j) |
$
is at most of order 
\[
O(n^m) +o(n^{2 m +2-k}). 
\]
Since $\varDelta(h \times h,r) =o( r^{-4})$, the growth rate is 
\begin{align*}
 \ev &\left( T_n \right)^2= \frac{O\left(n^m) +o(n^{2m-2}\right) }{ n^{2m-2}} =\begin{cases}
O(1)   & m=2\\
o(1) & m >2
\end{cases}
\end{align*}
For $m=2$ we have assumed existence of an  extra term $o(1)$,  which concludes the proof.
\end{proof}

We next  prove that the asymptotic distribution of a $V$-statistic depends on number of terms in the  Hoeffding decomposition that are equal to zero.

\begin{lemma}
\label{lem:equivVanila}
Let $h$ be a core with $m$ arguments. If $h_0=h_1=0$, and for all $c>2$  component  $(h_c \times h_c,Z_{t})$ is  of type $\varDelta$, with $\varDelta(h_c \times h_c,r) = o( r^{-4})$ then    
\begin{align*}
 \lim_{n \to \infty} \left( n V_n(h) -  \binom m 2  n V_n(h_2)  \right) \overset{L_2}{=} 0
\end{align*}
\end{lemma}

\begin{proof}
Using  Hoeffding decomposition we  write the core  $h$ as a sum of the components $h_c$ ,
\begin{align*}
  n V_n(h) =& n V_n(h_m) + \binom m 1 n V_n(h_{m-1}) + ... \\ 
  &+ \binom {m} {m-2} n V_n(h_{2}) + \binom {m} {m-1} n V_n(h_{1})+h_0.
\end{align*}
$h_0=0$ and  $h_1=0$. By  Lemma \ref{lem:higherVstats}, for $c \geq 3$, $n V_n(h_{c})$  converges to zero in mean squared. To see that it suffices to put $Q=1$ and verify that  $(h_c \times h_c,Z_t)$ is of $\varDelta$ type, which  is explicitly assumed.
\end{proof}

Before we study the asymptotic distribution of a bootstrapped statistic $B_n$ we need to sate three simple lemmas that will be frequently used.  
\begin{lemma}
\label{lem:meanWi}
If $W_i$ is a bootstrap process then
\begin{align*}
\lim_{n \to \infty} \frac {l_n}{ n} \sum_{i=1}^n W_i \overset{L_2}{=} 0.
\end{align*}
\end{lemma}
\begin{proof}
By the definition of $W_i$, $\ev (\sum_{i=1}^n W_i)^2 \leq n 2\sum_{r=1}^n Cov(W_0,W_r)=  nO(l_n)$, where $\sum_{r=1}^n Cov(W_0,W_r)=  O(l_n)$ follows from bootstrap assumption.  Also, by the  Bootstrap assumptions we have $\lim_{n \to \infty} \frac {l_n^3}{n^2} =0 $. Therefore $\frac{1} {n} \sum_{i=1}^{n}W_i$ converges to zero in mean squared.
\end{proof}

\begin{lemma}
\label{stmt:obviousD}
If $\{W_i\}$ is a bootstrap process then
\begin{align*}
\sum_{i=1}^n \tilde W_i = \sum_{i=1}^n  \left( W_i - \frac 1 n \sum_{j=1}^n  W_j \right) = 0. 
\end{align*}
\end{lemma}

\begin{lemma}
\label{lem:summingLema}
Let $f$ be a  function and let $j=\{j_1,\ldots,j_q\}$ be a subset of $\{1,\ldots,m\}$. Then
\begin{align*}
\sum_{i \in N^m} f(Z_{i_{j_1}},...,Z_{i_{j_q}})= n^{m-q} \sum_{i \in N^q} f(Z_{i_1},...,Z_{i_q})
\end{align*}
\end{lemma}
\begin{proof}
 Each element $f(Z_{i_{j_1}},...,Z_{i_{j_q}})$ is repeated exactly $n^{m-q}$ times.
\end{proof}

We now prove an analogue of the Lemma \ref{lem:equivVanila} for bootstrapped statistics $B$.

\begin{lemma}
\label{lem:equivBoot}
Let $h$ be a core of a $m$ arguments and let $Q_i$ denote  $W_i$ or  $\tilde W_i$. If  
\begin{align*}
\frac{1} {n^2}  \sum_{i \in N^2}   Q_{i_1} Q_{i_2} h_0 &=0, \\
\frac{1} {n^m}  \sum_{i \in N^m}  \sum_{1 \leq j \leq m } Q_{i_1} Q_{i_2} h_1(Z_{i_j}) &=0.
\end{align*}
and $(h_c,Z_{t})$ for $c>2$  are  of type $\varDelta$, with $\varDelta(h_c \times h_c,r) = o( r^{-4})$ then    
\begin{align*}
  \lim_{n \to \infty} \left( n B(h) -  \binom m 2  n B(h_2)  \right) \overset{L_2}{=} 0
\end{align*}
\end{lemma}

\begin{proof}
Where  it is necessary, we check claims for both $W_i$ and  $\tilde W_i$ separately. We will frequently use the fact that 
$
 \frac{l_n}{n} \sum_{i=1}^{n}Q_i, \frac{1}{n} \sum_{i=1}^{n}Q_i
$
converge to zero in mean square.

Using Hoeffding decomposition we  write core  $h$ as a sum of components $h_c$ (the ones with $h_0,h_1$ are equal to zero and therefore omitted)
\begin{align*}
 n B_1(h) = \frac{1} {n^{m-1}}  &\sum_{i \in N^m}  \Big[ Q_{i_1} Q_{i_2}   h_m(Z_{i_1},...,Z_{i_m})  + \\ 
 &\sum_{1 \leq j_1 < ...<j_{m-1} \leq m } Q_{i_1} Q_{i_2} h_{m-1}(Z_{i_{j_1}},...,Z_{i_{j_{m-1}}})   + ... + \\
 &\sum_{1 \leq j_1 < j_2 \leq m } Q_{i_1} Q_{i_2} h_2(Z_{i_{j_1}},Z_{i_{j_2}}) \Big].
\end{align*}
Consider the sum associated with $h_c$
\begin{align}
\label{eq:sumfortwo}
\frac{1} {n^{m-1}}  \sum_{i \in N^m}  \sum_{1 \leq j_1 < ...<j_c \leq m } Q_{i_1} Q_{i_2} h_c(Z_{i_{j_1}},...,Z_{i_{j_c}}).
\end{align}
We will show that for almost all fixed  $j_1 < \cdots < j_c$ the sum \ref{eq:sumfortwo} converges to zero.

Suppose  $j_1 >2$. The sum \ref{eq:sumfortwo} can be written
\begin{align*}
&\frac{1} {n^{m-1}}  \sum_{i \in N^m}   Q_{i_1} Q_{i_2} h_c(Z_{i_{j_1}},...,Z_{i_{j_c}})	 \overset{L.\ref{lem:summingLema}}{=\joinrel=}  
\frac{1} {n^{c+1}}  \sum_{i \in N^{c+2}}   Q_{i_1} Q_{i_2} h_c(Z_{i_3},...,Z_{i_{c+2}})  \\
=& \left( \frac{1}{n^{c-1}}   \sum_{ i \in N^c} h_c(Z_{i_1},...,Z_{i_c}) \right) \left( \frac{1}{n} \sum_{i=1}^{n}Q_i \right)^2  = \frac{n}{l_n}V_n(h_c)  \left( \frac{l_n}{n} \sum_{i=1}^{n}Q_i \right)^2 .  
\end{align*}
By   Lemma \ref{lem:higherVstats}, for $c \geq 3$, $\frac{n}{l_n} V_n(h_{c})$  converges to zero in mean squared. Indeed, it is sufficient to put $G_i=1$ and $T_n = n V_n(h_{c})$ and notice that $\frac{n}{l_n} V_n(h_{c}) = \frac{1}{l_n} = o(1)T_n$, since $l_n \to \infty$. Consequently, since   $(\frac{1}{n} \sum_{i=1}^{n}Q_i)^2$ converges to zero in mean square  \ref{lem:meanWi}, the product, converges to zero in mean square i.e.
\[
 V_n(h_c)  \left( \frac{1}{n} \sum_{i=1}^{n}Q_i \right)^2  \overset{L_2}{\to} 0
\]

Suppose  $j_1 = 2$. The sum  \ref{eq:sumfortwo} can be written
\begin{align}
 \label{eq:h2eq1}
 \begin{split}
&\frac{1} {n^{m-1}}  \sum_{i \in N^m}  Q_{i_1} Q_{i_2} h_c(Z_{i_2},...,Z_{i_{j_c}})  
\overset{L.\ref{lem:summingLema}}{=\joinrel=} \frac{1} {n^c}  \sum_{i \in N^{c+1}}   Q_{i_1} Q_{i_2} h_c(Z_{i_2},\cdots,Z_{i_{j_c}}) = \\
& \left( \frac{1}{l_n n^{c-1} } \sum_{i \in N^c} Q_{i_1}  h_c(Z_{i_1}, \cdots ,Z_{i_c}) \right) \left( \frac {l_n}{n} \sum_{i=1}^{n}Q_i \right).
\end{split}
\end{align}
The latter expression $\frac {l_n}{n} \sum_{i=1}^{n}Q_i$ converges to zero in mean square. The former expression can be further decomposed
\begin{align*}
& \frac{1}{ l_n} n^{-c+1} \sum_{i \in N^c} Q_{i_1}  h_c(Z_{i_1}, \cdots ,Z_{i_c}) = \frac 1 4 \frac{1}{l_n} (T_{+}-T_{-}) \text{ where,} \\   
\frac{1}{l_n} T_{-} &=   \frac{1}{l_n} n^{-c+1} \sum_{i \in N^2} (Q_{i_1}-1)h_c(Z_{i_1}, \cdots ,Z_{i_c})(Q_{i_2}-1), \\
\frac{1}{l_n} T_{+} &= \frac{1}{l_n} n^{-c+1} \sum_{i \in N^2}  (Q_{i_1}+1)h_c(Z_{i_1}, \cdots ,Z_{i_c})(Q_{i_2}+1),
\end{align*}
We  use Lemma \ref{lem:higherVstats} for $\frac{1}{l_n} T_{+}$ and $\frac{1}{l_n} T_{-}$, to show that they converge to zero. We  need to check that 
\[
 \sup_{i } E ( Q_{i}+/-1)^4   <\infty
\]
If $Q_i = W_i$ this follows from the Bootstrap assumptions $\sup_{n} \sup_{i \leq n} \ev W_{i,n}^{4} < \infty$. If $Q_i = \tilde W_i$ we check that  
\[
 E (\frac 1 n \sum_{i=1}^n W_i)^4  \leq \sup_{n} \sup_{i \leq n} \ev W_{i,n}^{4},
\]
and so $\leq  \sup_{i} \ev (\tilde W_i)  < \infty$. Now we conclude that both $\frac{1}{l_n} T_{+}$ and $\frac{1}{l_n} T_{-}$ converge to zero. Therefore their sum (even though they are not independent) converges to zero. 

Suppose $j_1=1$ and $j_2>2$. This case is identical to the previous case, up to swapping $i_1,i_2$ in the equation \ref{eq:h2eq1}.  

Finally, suppose $j_1=1$ and $j_2=2$ and $c>2$. The sum  \ref{eq:sumfortwo} can be written
\begin{align*}
\frac{1} {n^{m-1}}  \sum_{i \in N^m}  Q_{i_1} Q_{i_2} h_c(Z_{i_1},Z_{i_2},...,Z_{i_{j_c}})  \overset{L.\ref{lem:summingLema}}{=\joinrel=} \frac{1} {n^c}  \sum_{i \in N^{c+1}}   Q_{i_1} Q_{i_2}  h_c(Z_{i_1},Z_{i_2},...,Z_{i_{j_c}})
\end{align*}
We  again use Lemma \ref{lem:higherVstats} to see that this sum converges to zero in mean squared (we checked the assumptions above). We have proved that 
\begin{align*}
  \lim_{n \to \infty} \left( n B(h) -  \binom m 2  n B(h_2)  \right) \overset{L_2}{=} 0
\end{align*}
\end{proof}
So far we  avoided expressing results in terms of $\tau$-mixing and degeneracy of a core, now we relate $\varDelta$ formalism to those concepts. We start with a technical lemma. 
\begin{lemma}
\label{stm:LipAndBound}
 If $h$ is a  Lipschitz continuous core then its components are also Lipschitz continuous.
\end{lemma}
\begin{proof}
The auxiliary function  used in the Hoeffding decomposition
\begin{align*}
g_c(z_1,...z_c) = \ev h(z_1,...,z_c,Z_{c+1}^*,...,Z_{m}^*).  
\end{align*}
is Lipschitz, since $h$ is Lipschitz continuous.
\begin{align*}
|&g_c(z_1,...z_c) - g_c(z_1',...z_c')| \\
 &\leq \left| \int    [h(z_1,...,z_c,z_{c+1},...,z_{m}) - h(z_1',...,z_c',z_{c+1},...,z_{m}) ] dP(z_{c+1}) \cdots dP(z_m)\right| \\
 &\leq \left| \int    Lip(h) \left(  \sum_{i=1}^c | z_i - z_i'| + \sum_{i=c+1}^m | z_i - z_i|  \right)  dP(z_{c+1})  \cdots dP(z_m) \right| \\
  &\leq \left| \int    Lip(h) \left(  \sum_{i=1}^c | z_i - z_i'|   \right)   dP(z_{c+1})  \cdots dP(z_m) \right| \\
& = | Lip(h)   \sum_{i=1}^c | z_i - z_i'| \int  dP(z_{c+1})   \cdots dP(z_m)  |  \\
& = | Lip(h)   \sum_{i=1}^c | z_i - z_i'|  |.  \\
\end{align*}
$h_0$ is obviously Lipschitz continuous. If $h_{k}$ for $k<c$ are Lipschitz continuous then, since $g_c$ is Lipschitz continuous, $h_c$ is also Lipschitz continuous as a sum of Lipschitz continuous functions.
\end{proof}
\begin{lemma}
\label{lem:disentangle}
Let $\left\{ Z_{t}\right\} $ be a $\tau$-dependent stationary process and $h$ be a Lipschitz  core of $m$ arguments, If   for all $c>0$
$(h_c \times h_c,Z_t)$ and $(h,Z_t)$  are of type $\varDelta$ with the rate  $O(\tau(d))$ then 
\[
 \varDelta(h,d) =\varDelta(h_c \times h_c,d)  = O(\tau(d))
\]
\end{lemma}

\begin{proof}
Let $f = h_c \times h_c$ or $f=h$. $f$ is canonical and Lipschitz continuous (if $f = h_c \times h_c$ it follows from  Lemma \ref{stm:LipAndBound}).   Suppose $i_r$ is  the isolating index. Further suppose there are $k$  indexes $a_1,\cdots ,a_k$ smaller than $i_r$ and $m-k-1$ indexes greater than $i_r$, namely $a_{k+2}, \cdots , a_m$. In this  notation $a_{k+1}=i_r$.   

Let us partition the vector $\left(Z_{i_{1}},\ldots,Z_{i_{m}}\right)$ into three parts:
\begin{equation*}
A =  \left(Z_{a_{1}},\ldots,Z_{a_{k}}\right),\; B=Z_{a_{k+1}},\; C=\left(Z_{a_{k+2}},\ldots,Z_{a_{m}}\right).
\end{equation*}
where $a_{k+1}$ is the isolating index. If $k=0$, $A$ is empty and if $k=m-1$, $C$ is empty but this does not change our arguments below. Using
Lemma \cite[Lemma 5.3]{dedecker2007weak}, we will construct $B^{*}$ and $C^{**}$ that are independent
of $A$ and independent of each other and 
\begin{equation}
\ev \left\Vert \left(A,B,C\right)-\left(A,B^{*},C^{**}\right)\right\Vert _{1} =  O(\tau\left(w\right)), \label{eq: ABCLemma_property1}
\end{equation}
where $w$ is an isolating distance \footnote{ \cite[Lemma 5.3]{dedecker2007weak}   assumes that there exists a random variable $\delta$ independent of the vector $(A,B,C)$. This assumption is important only if CDF of the vector is not continuous, we can assume that our space is endowed with such $\delta$.}.  Let $D=(B,C)$ The \cite[Lemma 5.3]{dedecker2007weak}  guarantees that there exist $D^*$ independent of $A$, such that 
\begin{align*}
 \|& \ev d(D,D^*) | \sigma(A) \|_1 = \ev |  \ev d(D,D^*) | \sigma(A) | \\
 &= \ev   (\ev d(D,D^*) | \sigma(A)) = \ev d(D,D^*) = O(\tau(w)),
\end{align*}
where $d$ is the $L_1$ distance on Euclidean space (non-negativity justifies dropping absolute value). By definition of  $\tau$-mixing, $\tau(w) \geq \tau(\sigma(A),D )$. Since  $D^*=(B^*,C^*)$ has the same distribution as $D$ (in particular it has the same $\tau$ dependence structure) we use the lemma again to construct $C^{**}$, independent of $A$ and $B^*$, such that 
\[
  \ev d(C,C^{**}) =  O(\tau(w)).
\]
By the triangle inequality we obtain equation \ref{eq: ABCLemma_property1}. 
\begin{align*}
&\ev  d \big( (A,B,C) - (A,D^*) + (A,D^*) - (A,B^*,C^{**}) \big) \leq \\
&\ev d \big( (A,B,C) - (A,D^*)\big) + \ev d\big((A,D^*) - (A,B^*,C^{**}) \big) =\\
&\ev d(D,D^*) + \ev d(C,C^{**}) = O(\tau(w)).
\end{align*}
Since $B^{*}$ is a singleton, independent of both $A$ and $C^{**}$, by degeneracy of $f$  
\begin{equation}
\ev f(A,B^{*},C^{**})=0.\label{eq: ABCLemma_property2}
\end{equation}
Note that $ f(A,B^{*},C^{**})$ is just a shorthand, random variables $A,B^{*},C^{**}$ are  inserted in the right order. Thus, we have that
\begin{align*}
\left|\ev f\left(Z_{i_{1}},\ldots,Z_{i_{m}}\right)\right| & \leq  \ev \left|f\left(A,B,C\right)-f\left(A,B^{*},C^{**}\right)\right|+\left|\ev f(A,B^{*},C^{**})\right|\\
 & \leq  \text{Lip}(f)\ev \left\Vert \left(A,B,C\right)-\left(A,B^{*},C^{**}\right)\right\Vert _{1}+0\\
&= O(\tau(w)).
\end{align*}
\end{proof}

Finally we can prove   Theorem \ref{th:mainOne}. 


\begin{proof}
In the proof we are going to use \cite{leucht_dependent_2013}[Theorems 2.1, 3.1], which characterise asymptotic properties of $nV_n(h_2)$ and $n B(h_2)$. Both  theorems use similar set of assumptions which we verify upfront.  \\
\textit{Assumption A2.}\begin{itemize}
 \item (i)  $h_2$ is one-degenerate and symmetric - this follows from the Hoeffding decomposition;
 \item (ii) $h_2$ is a kernel - is one of the assumptions of this theorem;
 \item (iii) $\ev h_2(Z_1,Z_1) < \infty$ -- follows from $\sup_{i \in N^6}|\ev h(Z_i) |<\infty$ ;
 \item  (iv) $h_2$ is Lipschitz continuous - follows from the Lemma \ref{stm:LipAndBound}.
\end{itemize}
\textit{Assumption B1, A1.} Assumption $B1$, $\sum_{r=1}^n r^2 \sqrt{\tau(r)} < \infty$, is the same as ours, assumption $A1$, $\sum_{r=1}^n  \sqrt{\tau(r)} < \infty$ is implied.\\
\textit{Assumption B2.} This assumption about the bootstrap process $W_t$ is the same as our Bootstrap assumptions.

Denote by $V$ the weak limit of $n V_n(h_2)$, which exits by the  \cite{leucht_dependent_2013}[Theorem 2.1],  and let $\mathcal{F} = \sigma(Z_1, \cdots , Z_n)$.   By \cite[Theorem 3.1]{leucht_dependent_2013}, since the distribution of  $V$ is continuous, we have 
\begin{align*}
\sup_{x \in R} &\left| P(nB_n(h_2)  < x|\mathcal{F}) -   P(V<x) \right| \to 0  \\
\end{align*}
in probability. We show that $nB_n(h_2)$ converges  to $V$ weakly, by showing  pointwise convergence  of CDF  
\begin{align*}
 \lim_{n \to \infty} &P(nB_n(h_2)  < x) =  \lim_{n \to \infty} \ev P(nB_n(h_2)  < x|\mathcal{F}) \\
 &=  \ev  \lim_{n \to \infty} P(nB_n(h_2)  < x|\mathcal{F})  = \ev P(V<x) =P(V<x) 
\end{align*}
To change the order of limit and expectation we have dominated convergence Theorem, justified since  $P(nB_n(h)  < x|\mathcal{F})$  are bounded by 1.
The difference $n(B_n(h) - V_n(h))$ is
\begin{align*}
 n \left (B_n(h) -  \binom m 2 B_n(h_2) \right) + \binom m 2 \left (n B_n(h_2) -V\right)+ \left (\binom m 2 V - nV_n(h)\right)
 \end{align*}
By  Lemma \ref{lem:equivBoot} and Lemma \ref{lem:equivVanila} respectively, both 
$$n (B(h) -   \binom m 2  B(h_2)) , n (V_n(h) - n \binom m 2 V_n(h_2))$$
converge to zero in mean square. We check assumptions: since $Z_t$ is tau mixing and $h$ is Lipschitz continuous, by Lemma \ref{lem:disentangle} all self products of components and $Z_t$, $(h_c \times h_c,Z_t)$ for $c>0$, are $\varDelta$ type of order $\tau(r)$, of order at least  $o(r^{-4})$ (since $\sum_{r=1}^n r^2 \sqrt{\tau(r)} < \infty$). Since $h$ is one degenerate,  first and zero component $h_0,h_1$ are equal to zero (and so are $B(h_0),B(h_1)$).  

This shows that $nB_n(h_2)$ converges weakly to $V$. 
\end{proof}  

\subsection{Proof of Theorem \ref{th:mainThree}}

\begin{proof}
Using  Hoeffding decomposition we  write the core  $h$ as a sum of the components $h_c$ ,
\begin{align*}
  n V_n(h) =& n V_n(h_m) + \binom m 1 n V_n(h_{m-1}) + ... \\ 
  &+ \binom {m} {m-2} n V_n(h_{2}) + \binom {m} {m-1} n V_n(h_{1})+h_0.
\end{align*}
By the  Lemma \ref{lem:higherVstats}, for $c \geq 1$, $V_n(h_{c})$  converges to zero in probability. The sum associated with $h_1$ is
\[
V_n(h_1) = \frac 1 n \sum_{i=1}^{N} h_1(Z_i).
\]
By Lemma \ref{lem:disentangle}  $(h_1 \times h_1,Z_t)$ is $\varDelta$ type of order  $o(r^{-4})$. Using Lemma \ref{lem:higherVstats} we get the growth rate of 
$ \ev (V_n(h_1))^2= O(\frac 1 n)$, thus $V_n(h_1)$ converges in mean square to zero.
\end{proof}

\subsection{Proof of Theorem \ref{th:mainTwo}}
\label{sec:prMainTwo}


\begin{proof}
 We  show that the second non central moment of $B_1$ converges to $0$. The second non central moment is 
\begin{align*}
 \ev B_1 &= \ev \frac {1} {n^{2m}} \sum_{i \in N^{2m}}  W_{i_1} W_{i_2}  W_{i_{m+1}} W_{i_{m+2}} \ev h(Z_{i_1},...,Z_{i_m})  h(Z_{i_{m+1}},...,Z_{i_{2m}})\\
              &=\frac {1} {n^{2m}} \sum_{i \in N^{2m}} \ev W_{i_1} W_{i_2}  W_{i_{m+1}} W_{i_{m+2}}  \ev h(\cdots) h(\cdots)   \\
               & \leq  C \ev  \frac {1}  {n^4} \sum_{i \in N^4}  |\ev W_{i_1} W_{i_2}  W_{i_{m+1}} W_{i_{m+2}}|\\
              & = C \ev  \left( \frac {1}  {n}  \sum_{i=1}^n W_i \right)^4.
\end{align*}
The inequality in the third line follows from the fact that correlations of the bootstrap process $W_i$ are positive (Bootstrap assumption) and 
$$C = \sup_{n} \sup_{i \in N^m}\ev h(Z_{i_1},...,Z_{i_m})  h(Z_{i_{m+1}},...,Z_{i_2m}),$$ 
is finite. By Lemma \ref{lem:meanWi} 
\[
 \frac {1}  {n}  \sum_{i=1}^n W_i  \to 0,
\]
and therefore  $\ev C \left( \frac {1}  {n}  \sum_{i=1}^n W_i \right)^4 \to 0$.

We now prove that $o(n) B_2(h)$ converges to zero. Using Hoeffding decomposition  we write core  $h$ as a sum of components $h_c$ and $h_0$  
\begin{align}
\label{eq:bootstrapedOne}
 &n B_2(h) = \frac{1} {n^{m-1}}  \sum_{i \in N^m}  \Big[h_0  \tilde W_{i_1} \tilde W_{i_2} + \sum_{1 \leq j \leq m } \tilde W_{i_1} \tilde W_{i_2} h_1(Z_{i_j})    \\ 
 &\sum_{1 \leq j_1 < j_2 \leq m } \tilde W_{i_1} \tilde W_{i_2} h_2(Z_{i_{j_1}},Z_{i_{j_2}}) + ... +  \tilde W_{i_1} \tilde W_{i_2}   h_m(Z_{i_1},...,Z_{i_m}) \Big].
\end{align}
We examine terms of the above sum starting form the one with $h_0$ - it is equal to zero
\begin{align*}
\frac{1} {n^{m-1}}  \sum_{i \in N^m}  h_0  \tilde W_{i_1} \tilde W_{i_2}   \overset{L.\ref{lem:summingLema}}{=\joinrel=} \frac 1 n h_0 \sum_{i \in N^2} \tilde W_{i_1} \tilde W_{i_2} = \frac 1 n h_0 \left( \sum_{i=1} \tilde W_i \right)^2  \overset{L.\ref{stmt:obviousD}}{=\joinrel=} 0.
\end{align*}  
Term with $h_1$ is zero as well, to see that fix $j$ and consider 
\begin{align*}
T_{j} = \frac{1} {n^{m-1}}  \sum_{i \in N^m}  \tilde W_{i_1} \tilde W_{i_2} h_1(Z_{i_j}).  
\end{align*}  
If $j=1$ then
\begin{align*}
T_{1} \overset{L.\ref{lem:summingLema}}{=\joinrel=} \frac{1} {n}  \sum_{i \in N^2}  \tilde W_{i_1} \tilde W_{i_2} h_1(Z_{i_1}) =  \frac{1} {n}  \left( \sum_{i=1}^n  \tilde W_i h_1(Z_i) \right) \left( \sum_{i=1} \tilde W_i \right) \overset{L.\ref{stmt:obviousD}}{=\joinrel=} 0.
\end{align*}
If $j=2$ the same reasoning holds and if $j>2$
\begin{align*}
T_{j} \overset{L.\ref{lem:summingLema}}{=\joinrel=} \frac{1} {n^2}  \sum_{i \in N^3}  \tilde W_{i_1} \tilde W_{i_2} h_1(Z_{i_3}) =  \frac{1} {n}  \left( \sum_{i=1}^n h_1(Z_i) \right) \left( \sum_{i=1} \tilde W_i \right)^2 \overset{L.\ref{stmt:obviousD}}{=\joinrel=} 0.
\end{align*}
By Lemma \ref{lem:equivBoot}, since $B(h_0)=B(h_1)=0$, $(nB(h) - \binom m 2 nB(h_2)) \to 0$ in mean square and the only term that remains is 
\begin{align*}
T_n = \frac{1} {n}  \sum_{i,j \in N}  \tilde W_{i} \tilde W_j h_2(Z_i,Z_j)
\end{align*}
 Now we can use the Lemma \ref{lem:higherVstats} to show that $o(1) T_n$ converges to zero. 
 \end{proof}

\subsection{Proof of Proposition \ref{prop:mmd} }
  \label{sub:prop:mmd}
\begin{proposition*}
 Let $k$ be bounded and Lipschitz continuous, and let $\left\{ X_t \right\}$ and $\left\{ Y_t \right\}$ 
 both be $\tau$-dependent with coefficients $\tau(i) =  O(i^{-6-\epsilon})$, but independent of each other. Further, let $n_x=\rho_x n$ and $n_y=\rho_y n$ where $n=n_x+n_y$. Then, under the null hypothesis $P_x=P_y$, $\varphi\left(\rho_x \rho_y n\widehat{\text{MMD}}_k, \rho_x \rho_y n\widehat{\text{MMD}}_{k,b}\right)\to 0$ in probability as $n\to\infty$, where $\varphi$ is the Prokhorov metric.
\end{proposition*}
  \begin{proof}
  Since $\widehat{\text{MMD}}_k$ is just the MMD between empirical measures
using kernel $k$, it must be the same as the empirical MMD $\widehat{\text{MMD}}_{\tilde k}$ with centred kernel $\tilde{k}(x,x')=\left \langle k(\cdot,x)-\ev k(\cdot,X), k(\cdot,x')-\ev k(\cdot,X) \right \rangle_{\Hk}$ according to \cite[Theorem 22]{SejSriGreFuk13}. Using the Mercer expansion, we can write
\begin{align*}
\rho_x \rho_y n\widehat{\text{MMD}}_k & = \rho_{x}\rho_{y}n\sum_{r=1}^{\infty}\lambda_{r}\left(\frac{1}{n_{x}}\sum_{i=1}^{n_{x}}\Phi_{r}(x_{i})-\frac{1}{n_{y}}\sum_{j=1}^{n_{y}}\Phi_{r}(y_{j})\right)^{2}\\
 & = \sum_{r=1}^{\infty}\lambda_{r}\left(\sqrt{\frac{\rho_{y}}{n_{x}}}\sum_{i=1}^{n_{x}}\Phi_{r}(x_{i})-\sqrt{\frac{\rho_{x}}{n_{y}}}\sum_{j=1}^{n_{y}}\Phi_{r}(y_{j})\right)^{2},
\end{align*}
where $\{\lambda_r\}$ and $\{\Phi_r\}$ are the eigenvalues and the eigenfunctions of the integral operator $f\mapsto \int f(x)\tilde k(\cdot,x)dP_x(x)$ on $L_2(P_x)$. Similarly as in \cite[Theorem 2.1]{leucht_dependent_2013}, the above converges in distribution to $\sum_{r=1}^\infty \lambda_r Z_r^2$, where $\{Z_r\}$ are marginally standard normal, jointly normal and given by $Z_r=\sqrt{\rho_x}A_r-\sqrt{\rho_y}B_r$. $\{A_r\}$ and $\{B_r\}$ are in turn also marginally standard normal and jointly normal, with a dependence structure induced by that of $\{X_t\}$ and $\{Y_t\}$ respectively. This suggests individually bootstrapping each of the terms $\Phi_{r}(x_{i})$ and $\Phi_{r}(y_{j})$, giving rise to 
\begin{align*}
\widehat{\text{MMD}}_{\tilde k, b}&=\sum_{r=1}^{\infty}\lambda_{r}\left(\frac{1}{n_{x}}\sum_{i=1}^{n_{x}}\Phi_{r}(x_{i})\tilde W_i^{(x)}-\frac{1}{n_{y}}\sum_{j=1}^{n_{y}}\Phi_{r}(y_{j})\tilde W_j^{(y)}\right)^{2}\\
{}&=\quad\frac{1}{n_x^2}\sum_{i=1}^{n_x}\sum_{j=1}^{n_x}\tilde W_i^{(x)}\tilde W_j^{(x)}\tilde k(x_i,x_j)-\frac{1}{n_x^2}\sum_{i=1}^{n_y}\sum_{j=1}^{n_y}\tilde W_i^{(y)}\tilde W_j^{(y)}\tilde k(y_i,y_j)\\
{}&\qquad-\frac{2}{n_x n_y}\sum_{i=1}^{n_x}\sum_{j=1}^{n_y}\tilde W_i^{(x)}\tilde W_j^{(y)}\tilde k(x_i,y_j). 
\end{align*}
Now, since $\tilde k$ is degenerate under the null distribution, the first two terms (after appropriate normalization) converge in distribution to $\rho_x\sum_{r=1}^\infty \lambda_r A_r^2$ and  $\rho_y\sum_{r=1}^\infty \lambda_r B_r^2$ by \cite[Theorem 3.1]{leucht_dependent_2013} as required. 
The last term follows the same reasoning - it suffices to check part (b) of \cite[Theorem 3.1]{leucht_dependent_2013} (which is trivial as processes $\left\{ X_t \right\}$ and $\left\{ Y_t \right\}$ are assumed to be independent of each other) and apply the continuous mapping theorem to obtain convergence to $-2\sqrt{\rho_x\rho_y}\sum_{r=1}^\infty \lambda_r A_rB_r$ implying that $\widehat{\text{MMD}}_{\tilde k, b}$ has the same limiting distribution as $\widehat{\text{MMD}}_{k}$.
While we cannot compute $\tilde k$ as it depends on the underlying probability measure $P_x$, it is readily checked that due to the empirical centering of processes $\{\tilde W_t^{(x)}\}$ and $\{\tilde W_t^{(y)}\}$, $\widehat{\text{MMD}}_{\tilde k, b}=\widehat{\text{MMD}}_{k, b}$ holds and the claim follows. Note that the result fails to be valid for wild bootstrap processes that are not empirically centred.
\end{proof}

\section{Various comments}

\subsection{Time complexity}
The original HSIC and MMD tests for i.i.d. data, the computational cost of the wild bootstrap approach scales quadratically in the number of samples, and linearly in the number of bootstrap iterations (in the i.i.d. case, these were permutations of the data). The main alternative approaches are the lagged bootstrap of \cite{chwialkowski2014kernel}, which has the same scaling with data and number of bootstraps, and the spectrogram approach of \cite{besserve_statistical_2013} (note, however, that both these alternative approaches apply only to the independence testing case). The cost of \cite{besserve_statistical_2013} is comparable to our approach, however the statistical power of \cite{besserve_statistical_2013} was much weaker on the data we examined.

\end{document}